\documentclass[11pt, conference]{IAES2024}
\usepackage[top=1.5cm, bottom=1.5cm, left=2.4cm, right=2.4cm, includehead, includefoot, heightrounded]{geometry}

\usepackage{hyperref}
\hypersetup{pdftex,colorlinks=true,allcolors=blue}
\usepackage{hypcap}

\let\labelindent\relax 

\usepackage{booktabs}
\usepackage[shortlabels]{enumitem}
\usepackage{newclude}
\usepackage{blindtext}
\usepackage{mathptmx}
\usepackage[english]{babel}
\usepackage{amsthm}

\theoremstyle{definition}
\newtheorem{definition}{Definition}[section]

\theoremstyle{plain}
\newtheorem{theorem}{Theorem}[section]
\newtheorem{corollary}{Corollary}[theorem]
\newtheorem{lemma}[theorem]{Lemma}
\theoremstyle{remark}
\newtheorem*{remark}{Remark}
\usepackage{wrapfig}
\usepackage{titlesec}
\titleformat{\section}
{\normalfont\rmfamily\Large\bfseries}
  {\thesection}{1em}{}

 \titleformat{\subsection}
{\normalfont\rmfamily\bfseries}
  {\thesubsection}{1em}{}
\titlespacing*{\section} {0pt}{2.5ex}{0.6ex}
\titlespacing*{\subsection} {0pt}{1.5ex}{0.6ex}
\usepackage{fancyhdr} 
\usepackage{graphicx} 
\usepackage[textfont=bf, font = normalsize]{caption} 
\usepackage{floatrow} 
\usepackage{amssymb}
\usepackage{amsfonts}
\usepackage{amsmath}
\usepackage[square,numbers,sort&compress]{natbib}
\floatstyle{plaintop} 
\restylefloat{table} 
\usepackage{subfigure} 
\renewcommand{\headrulewidth}{0pt}

\makeatletter
\renewcommand{\thesection}{\arabic{section}}
\renewcommand{\thesubsection}{\thesection.\arabic{subsection}}
\renewcommand{\p@subsection}{}
\renewcommand{\thesubsubsection}{\thesubsection.\arabic{subsubsection}}
\renewcommand{\p@subsubsection}{}
\makeatother
\usepackage{xcolor}
\begin{document}

\newcommand{\bing}[1]{{\color{blue}{\small\bf\sf [bing: #1]}}}
\newcommand{\changnan}[1]{{\color{red}{\small\bf\sf [changnan: #1]}}}

\title{\LARGE{A Theory for Length Generalization in Learning to Reason}


\vspace{-8mm}
}
\author{\textbf{Changnan Xiao\textsuperscript{1} ~and ~Bing Liu\textsuperscript{2} 
}
\vspace{+2mm}
\\
\textsuperscript{1}\small{ChangnXX.github.io}\\
\textsuperscript{2} \small{Department of Computer Science, University of Illinois Chicago}\\
\small{changnanxiao@gmail.com ~~~~ liub@uic.edu}}
\maketitle
\thispagestyle{fancy} 
\pagestyle{plain}
\vspace{7mm}

\begin{abstract}
\textit{Length generalization} (LG) is a challenging problem in learning to reason. It refers to the phenomenon that when trained on reasoning problems of smaller lengths or sizes, the resulting model struggles with problems of larger sizes or lengths. Although LG has been studied by many researchers, the challenge remains. This paper proposes a theoretical study of LG for problems whose reasoning processes can be modeled as DAGs (\textit{directed acyclic graphs}). The paper first identifies and proves the conditions under which LG can be achieved in learning to reason. It then designs problem representations based on the theory to learn to solve challenging reasoning problems like \textit{parity}, \textit{addition}, and \textit{multiplication}, using a Transformer to achieve perfect LG. 
\end{abstract}

\section{Introduction}
\label{sec-intro}
Large language models (LLMs) 
have been shown to perform reasoning tasks remarkably well~\citep{brown2020language,suzgun2022challenging,saparov2022language,liu2023evaluating,xu2023large}.~However, evaluations also revealed some limitations. For example, LLMs 
often have difficulties in simple \textit{addition} and \textit{multiplication} of large numbers~\citep{nogueira2021investigating,qian2022limitations}. 
A popular solution to improve reasoning is to use  \textit{Scratchpad}~\citep{nye2021show} or \textit{Chain of Thought} ({CoT})~\cite{wei2022chain}. The idea is to add the \textit{intermediate steps} for each reasoning problem in the training data. For example, the training sample for calculating $3 + 2 \times 1$ may be presented as $3 + 2 \times 1 = 3 + 2 = 5$ rather than $3 + 2 \times 1 = 5$. CoT has been used to improve reasoning \cite{anil2022exploring,liu2023goat,lee2023teaching}. 
However, 
\cite{dziri2023faith} and \cite{kazemnejad2023impact} reported that even with detailed  CoT steps, 
the learned models still fail to generalize for several reasoning problems. For example, they showed that when trained with smaller problems, e.g., multiplication of two smaller numbers such as $1234 \times 135$ based on CoT training data, the model cannot generalize to larger problems (e.g., $235469 \times 44562$). This problem is called \textbf{\textit{length generalization}} (\textbf{LG})~\cite{anil2022exploring,zhang2022unveiling,kazemnejad2023impact}.    

The paper proposes a theory for LG in learning to reason whose step-by-step reasoning processes can be formulated as \textit{directed acyclic graph} (DAG) structures.\footnote{~DAGs have been popularly used to represent reasoning problems \citep{shao2022chaining,cao2021bottom,huang2022directed,feng2023towards}. Some reasoning problems cannot be modeled as DAGs, e.g., temporal and spatial reasoning problems.} We will not study the case where the CoT steps are not given but only the direct input and output (e.g., $3+2 \times 1 = 5$) are provided as it has been proven that this case isn't learnable in general~\cite{wies2023sub,feng2023towards,malach2023auto}. 
Our theory introduces a new notion of \textit{maximal input element distance} $R$ of the elements involved in each reasoning step and a new concept of $(n, r)$-consistency. We first prove that for a CoT formulation of a reasoning problem, if it has the property of $R < \infty$, it is learnable with LG. We further show that if a reasoning problem with $R = \infty$ but it is $(n, r)$-consistent, it is also learnable with LG. 
Empirical evidence is given to verify the theory using a vanilla Transformer to learn many challenging reasoning tasks, including parity, addition, and multiplication, to achieve perfect LG. 

\section{Overview of the Proposed LG Theory}
\label{sec.overview}

{\color{black}We show that \textbf{(1)} given the DAG structure of a reasoning task, the condition for learning the \textbf{causal function} (the function performing a single reasoning step), \textbf{(2)} given the DAG structure and a well-learned causal function, a \textbf{recursive formula} being able to solve the LG problem, and \textbf{(3)} given only unstructured sequence (e.g., $3+2 \times 1$), the condition for learning the structure (i.e., predicting the elements used in the next reasoning step) and solving the LG problem.

To illustrate our notations, we take the problem of calculating $3 + 2 \times 1$ with the CoT formulation, $3 + 2 \times 1 = 3 + 2 = 5$, as an example. 
A \textit{causal function} infers one step in the reasoning/calculation process as specified in CoT. In this example, $2 \times 1 = 2$ and $3 + 2 = 5$ are both one casual/reasoning step. 
In the arithmetic calculation on prime field $F_p$ with the above CoT formulations, the causal function is $f: \mathbf{X} \rightarrow F_p$, where $\mathbf{X} = F_p \times \{+, -, \times, /\} \times F_p$ is the \textbf{input space} of the causal function {($F_p=\{0, 1, \dots, p - 1\}$)}. Different CoT formulations may give different causal functions (see Sec.~\ref{sec:experiment}). 

\begin{figure}
\caption{
An example DAG. 
} 
\vspace{+2mm}
\center
\includegraphics[width=0.25\linewidth]{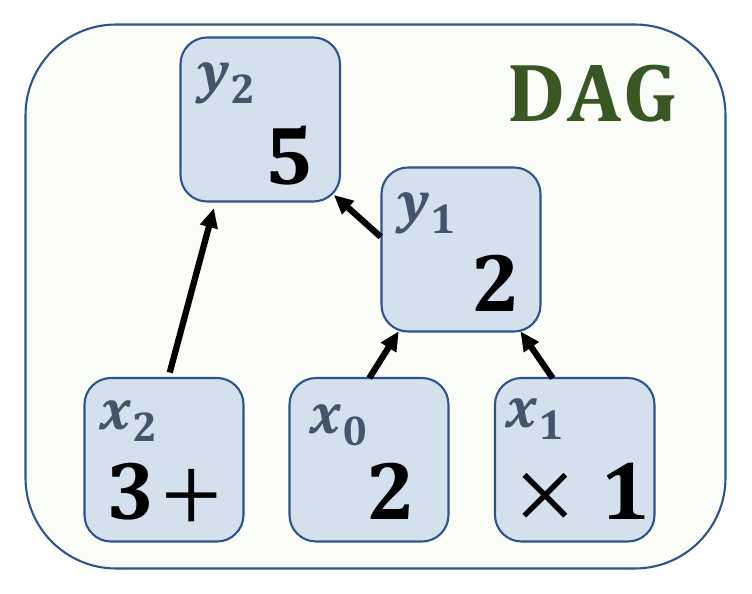}
\label{fig:example_mdp_dag}
\end{figure}

{The calculation process of $3 + 2 \times 1$ is represented as the DAG in Fig.~\ref{fig:example_mdp_dag}.}~Its topological sorting gives $x_0 = 2, x_1 = (\times, 1), x_2 = (3, +), y_1 = f(x_0, x_1) = 2 \times 1 = 2, y_2 = f(x_2, y_1) = 3 + 2 = 5$, where $f$ is the causal function. 

We define a \textit{recursive formula} as a function that recursively applies the causal function step-by-step to solve a reasoning problem according to its structure, e.g., $3 + 2 \times 1 = 3 + 2 = 5$ in Fig.~\ref{fig:example_mdp_dag}. In Sec.~\ref{sec: dag}, we assume that the DAG structure is given and we prove the following: 
\begin{itemize}
\item The causal function can be perfectly learned (called \textbf{well-learned}) only when the input space of the causal function is finite, $|\mathbf{X}| < \infty$. 
\item When evaluating with the DAG structure given, solving the reasoning problem by recursively applying a well-learned causal function can generalize from smaller training problem sizes to larger test problem sizes.
\end{itemize}
Sec. \ref{sec:cot_and_others} studies the realistic scenario where the DAG structure is unknown and only unstructured sequence data in CoT is given in training, e.g., a sequence of elements like $3 + 2 \times 1 = 3 + 2 = 5$. Sec. \ref{sec:cot_and_others} focuses on the following,
\begin{itemize} 
\item Learning to predict which elements should be the input to the well-learned causal function in the next reasoning step. We propose an important notion $R$, the \textbf{maximal input element distance} for a causal/reasoning step. For $3 + 2 \times 1$, we have $R=2$ because the elements that should be calculated next are in a window of length $3$, e.g., $2 \times 1$, where the maximal input elements distance (between $2$ and $1$) is $2$. 
\item Proving the sufficient condition for learning with LG is $R < \infty$ {for a class of problem, e.g., \textit{parity}.}
\end{itemize}
Sec.~\ref{sec: (n, r)-consistent} extends the theory to deal with $R = \infty$. 
\begin{itemize}
\item It introduces a more general condition, called \textbf{$(n, r)$-consistency}, to generalize $R$ and proves any reasoning problem with $(n, r)$-consistent can achieve LG regardless of whether $R < \infty$ or $R = \infty$. $R < \infty$ is a special case of $(n, r)$-consistent. 
\end{itemize}
Sec. \ref{sec:experiment} shows that straightforward CoT formulations of addition and multiplication do not satisfy the condition, and thus have the LG issue. However, with different CoT formulations, they are $(n, r)$-consistent and can be learned to achieve LG.

Learnability of reasoning problems using CoT has been reported in~\cite{wies2023sub,feng2023towards,malach2023auto} for neural networks. However, these studies are all under i.i.d and given problem length/size $N$. They do not cover LG. We discuss this in greater detail in Sec.~\ref{sec.related}. 
Our work can be seen as complementary to theirs in the context of neural networks (as our results are independent of specific learning paradigms or algorithms). Their learnability results still apply.~We add conditions under which the learned function can extrapolate to larger lengths than $N$ to achieve LG.



In summary, for learning to reason and overcome the {LG} issue, we propose \textbf{three sufficient conditions}, \textbf{(i)} the input space $\mathbf{X}$ of a causal/reasoning step of the reasoning problem is \textit{finite}, \textbf{(ii)} the problem should be solved \textit{recursively} based on CoT,
and \textbf{(iii)} its CoT formulation is $(n,r)$-\textit{consistent}.} 


\section{The Proposed LG Theory}

\subsection{Given the Directed Acyclic Graph (DAG)}
\label{sec: dag}

This subsection assumes that the DAG structure representing the reasoning process of a problem is given. 
A topological sorting can order the vertices of a DAG, where every edge leaves a front vertex and enters a later vertex. 

Denote a DAG as $G = (V, E)$, where $V$ is the set of all vertices and $E$ is the set of all edges. 
For easy understanding, we use $|G|$ $(= |V|)$ to represent the number of vertices in $G$. 
Denote $u \overset{e}{\rightarrow} v$ as when $v$ is reachable directly from $u$ by the edge $e$. 
Denote $p(v) = \{u \in V |\, \exists\, e \in E, u \overset{e}{\rightarrow} v\}$ to be all \textit{preceding} vertices (immediate in-neighbors) that can reach $v$ directly and 
the \textit{in-degree} of each vertex $v$ as $|p(v)|$. 

For a reasoning problem structured as a DAG, when $u \overset{e}{\rightarrow} v$, we say that $u$ is an \textbf{input vertex} and $v$ is the \textbf{causal vertex}. 
Since a causal vertex can also be an input vertex of some other vertices, we do not distinguish the domain and the range. With a slight abuse of notation, we also use $v$ to represent the value of vertex $v$. 
We simply denote $X$ to be both the domain and the range of all vertices. 
Let $f: X^{\sup |p(v)|} \rightarrow X$ be the \textbf{causal function}, which is $v = f(p(v))$. The \textbf{input space} of $f: X^{\sup |p(v)|} \rightarrow X$ is $\mathbf{X} = X^{\sup |p(v)|}$. 
Note that $G$ defines the structure of the reasoning process and $f$ defines the values. 
For any $G = (V, E)$, let $V = \{v_1, \dots, v_{|G|}\}$ be a topological sorting of the vertices in $G$. 
Note that the vertices with $|p(v)| = 0$ are pure input vertices of the graph.
We say
\begin{equation}
G_f(\{v_i \,|\, i \leq |G|,\, |p(v_i)| = 0\}) = (v_1, \dots, v_{|G|}), 
\end{equation}
\text{where} \
\begin{equation*}
\left\{
\begin{aligned}
&v_1 = f(p(v_1)), \\
&\dots \\
&v_n = f(p(v_n)),
\end{aligned}
\right.
\label{eq: def_dag}
\end{equation*}
instantiates or assigns values to all the vertices by $f$. 

\begin{figure}
\caption{
An example of notations. 
} 
\vspace{+2mm}
\center
\includegraphics[width=0.45\linewidth]{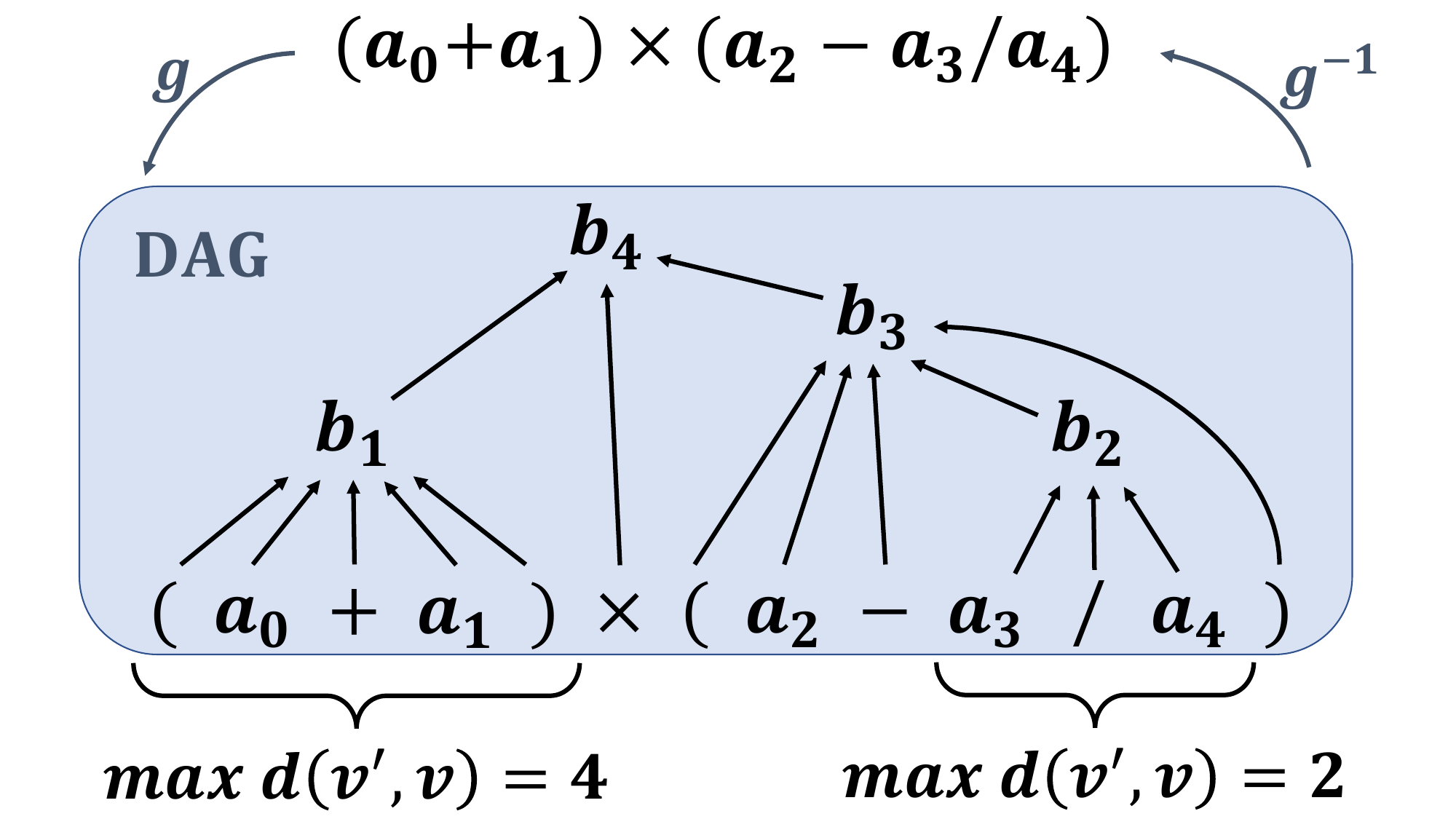}
\label{fig:example_dag}
\end{figure}


For example, arithmetic problems in a prime field $F_p$ ($= \mathbf{Z} / p\mathbf{Z}$) are DAGs. To calculate $(a_0 + a_1) \times (a_2 - (a_3 / a_4))$, we can use the steps:  $b_1 = f(a_0, +, a_1)$, $b_2 = f(a_3, /, a_4)$, $b_3 = f(a_2, -, b_2)$, and $b_4 = f(b_1, +, b_3)$, which define a DAG. The DAG defines the calculation steps and the causal function $f$ defines the arithmetic rules in $F_p$. When $f$ changes, the induced values on $G$ change. That's why we say $G_f$ is the induced graph of $G$ defined by $f$. 

We first study the property of $f$. 
Denote the training dataset of $f$ as $D$, which is a subset of the input space as 
\begin{equation}
D \subseteq \{(f(p(v)), p(v)) |\, v \in V\}.
\end{equation}
$D$ contains vertices from any $G$. For simplicity, denote $D \subseteq \mathbf{X} = X^{\sup |p(v)|}$ and $f(\mathbf{X}) = \{f(p(v))|\, p(v) \in \mathbf{X}\}$. We assume $|D| < \infty$, which is true for any training dataset. 

\vspace{+2mm}
\begin{theorem}
\label{thm: full_causal_function}
For $|X| < \infty$ and $\sup |p(v)| < \infty$, i.e., $|\mathbf{X}| < \infty$, if $D = \mathbf{X}$, 
then there exists an approximation function  $\Hat{f}: X^{\sup |p(v)|} \rightarrow X$,  s.t. $\Hat{f}(p(v)) = f(p(v)),\,\forall\, p(v) \in \mathbf{X}$.
\end{theorem}

The proof is given in \textit{Appendix}~\ref{sec.dynamic.causal}.~The assumption $\sup |p(v)| < \infty$ makes the domain of $f$ have finite dimensions, which guarantees the feasibility of learning $f$ with $|D| < \infty$. 
When $D \neq \mathbf{X}$, we have the following \textbf{negative} or \textbf{impossibility} corollaries (proofs are given in \textit{Appendix}~\ref{sec.dynamic.causal}).

\vspace{+2mm}
\begin{corollary}
\label{thm: not_full_causal_function}
For $|X| < \infty$ and $\sup |p(v)| < \infty$, i.e. $|\mathbf{X}| < \infty$, if $|f(\mathbf{X})| > 1$ and $D \neq \mathbf{X}$, then there exists an approximation function $\Hat{f}: X^{\sup |p(v)|} \rightarrow X$ s.t. $\Hat{f}(p(v)) = f(p(v)),\,\forall\,p(v) \in D$ and $\Hat{f}(p(v)) \neq f(p(v)),\,\forall\,p(v) \in  \mathbf{X} \backslash D$. 
\end{corollary}

\vspace{+2mm}
\begin{corollary}
\label{thm: not_full_causal_function_inf_x}
For $\max(|X|, \sup |p(v)|) = \infty$, i.e. $|\mathbf{X}| = \infty$, if $|f(\mathbf{X})| > 1$, for $\forall\, m > 0$, there exists an approximation function $\Hat{f}: X^{\sup |p(v)|} \rightarrow X$ s.t. $\Hat{f}(p(v)) = f(p(v)),\,\forall\,p(v) \in D$ and $|\{p(v) | \Hat{f}(p(v)) \neq f(p(v))\}| > m$. 
\end{corollary}
The causal function is guaranteed to be well-learned only when $|\mathbf{X}| < \infty$.

We now show that a DAG can only be solved \textbf{recursively}. 
\vspace{+2mm}
\begin{theorem}
For $|X| < \infty$ and $\sup |p(v)| < \infty$, if $D = X^{\sup |p(v)|}$, then there exists an approximation function $\Hat{f}: X^{\sup |p(v)|} \rightarrow X$, the DAG can be recursively solved, i.e., $\forall\,G = (V, E)$, $G_{\Hat{f}} = G_f$. 
\label{thm: recursive_dag_solved}
\end{theorem}
It shows that when the DAG structure of the problem is given and the causal function $f$ is well-learned, a problem of any length or size can be solved, which means achieving LG. See \textit{Appendix}~\ref{sec.recursive.direct} for the proof.

However, in practice, both the DAG and the causal function are unknown. Below, we deal with this realistic scenario. 

\subsection{Given Only the Unstructured Sequence}
\label{sec:cot_and_others}

We now study the realistic scenario where the DAG structure of the reasoning problem is unknown. We are given only the input sequence of the problem in a CoT representation. We discuss the possibility of learning to transform the input sequence of unstructured elements into a structured DAG. We achieve this recursively by learning to find input elements that should be reasoned next, like recursively doing topological sorting of the underlying DAG (we do not physically construct the DAG). For instance, $(a_0 + a_1) \times (a_2 - (a_3 / a_4))$ is a sequence of elements (including `(' and `)') and the ordering of calculation needs to be learned. 
Let us use this example to discuss more about CoT
and how it connects to DAG and recursive reasoning on DAG. The process of CoT (Fig.~\ref{fig:example_dag}) is 
$(a_0 + a_1) \times (a_2 - (a_3 / a_4)) 
= b_1 \times (a_2 - b_2)
= b_1 \times b_3 
= b_4$.

CoT does three things or has three sub-steps in each reasoning/causal step. Sub-step 1 decides a part of the topological order of the underlying DAG, i.e., finding the vertices that will be used in the next reasoning step, i.e., vertices in $\{p(v) |\,v \in W^0\}$ ($W^0$ are vertices to be valued in the next reasoning step, defined later), before the reasoning step.
Sub-step 2 applies the causal function (previously learned), e.g., $b_1 = f(a_0, +, a_1)$, $b_2 = f(a_3, /, a_4)$, etc. 
Sub-step 3 puts the result back into the unstructured sequence. Since the arithmetic example decreases the number of elements at each step, sub-step 3 looks natural. However, it's possible to have a reasoning problem whose underlying DAG induces multiple vertices with multiple elements after one reasoning step. The representation by the unstructured data of the result after one reasoning step shouldn't be ambiguous; otherwise, the reasoning problem itself is not well-defined. 

To define a well-defined reasoning problem, we need more notations. For a DAG $G = (V, E)$, we have defined $p(v) = \{u \in V |\, \exists\, e \in E, u \overset{e}{\rightarrow} v\}$ to be all vertices that can reach $v$ directly ($v$'s in-neighbors). Now we define $t(v) = \{w \in V |\, \exists\, e \in E, v \overset{e}{\rightarrow} w\}$ to be all vertices that can be reached from $v$ directly ($v$'s out-neighbors). We define $(p \circ t)(v) = \{v' \in V |\, \exists\, w \in V,\, \exists\, e_1, e_2 \in E, v \overset{e_1}{\rightarrow} w, v' \overset{e_2}{\rightarrow} w \}$ to be the vertices that can reach any causal vertex in $t(v)$.

Now we define what a well-defined reasoning problem is in an unstructured sequence. Denote the initial unstructured data to be a sequence of elements $s^0 = s_1^0 s_2^0 \dots s_{i_0}^0$. Let $g$ be the function that transforms the unstructured data into a structured DAG. We write $G^0 = (V^0, E^0) = g(s^0)$ (see Fig.~\ref{fig:example_dag}). We say a vertex is \textit{valued} if a value is assigned to the vertex. In $G^0$, all vertices with $|p(v)| = 0$ have been valued as they are pure input vertices, but all causal vertices with $|p(v)| > 0$ haven't been valued. After one reasoning step on $G^0$, the vertices $W^0 = \{v \in G^0|\, p(v) \subseteq \{v \in G^0|\, |p(v)| = 0\}\}$ are valued. Then the vertices $U^0 = \{v \in \bigcup_{v \in W^0} p(v)|\, t(v) \subseteq W^0 \}$, whose causal vertices are valued, become useless in the following reasoning steps. We denote $G^1 = G^0 \setminus U^0$ to be the subgraph of $G^0$ after removing all vertices in $U^0$ and all edges from $U^0$. To represent $G^1$, let $g^{-1}$ be the inverse function that maps a DAG back to the unstructured data (which may not be unique) (Fig.~\ref{fig:example_dag}). 
We write $s^1 = s_1^1 s_2^1 \dots s_{i_1}^1 = g^{-1} (G^1)$. 

We say the reasoning problem is \textit{well-defined} with the unstructured data if $g(g^{-1} (G)) = G,\,\forall\, G$. Note that $g(g^{-1} (G)) = G$ represents that the DAG is isomorphic after the composing transformation $(g \circ g^{-1})$. 

Now we describe the three sub-steps of CoT again. Sub-step 1 finds the input vertices $\{p(v) |\,v \in W^0\}$ and the next valued vertices $W^0 = \{v \in G^0|\, p(v) \subseteq \{v \in G^0|\, |p(v)| = 0\}\}$.
Since CoT only needs a one-step-forward subgraph instead of the entire graph. It's unnecessary to apply the original $g$. Instead, we define 
\vspace{-1mm}
\begin{align}
\Tilde{g}(s^0) \overset{def}{=} \{p(v) |\, v \in g(s^0), p(v) \subseteq \{v \in g(s^0)|\, |p(v)| = 0\}\},
\label{eq: input_elements}
\end{align}
which finds the combinations of input vertices that infer causal vertices for one causal step. 
Different from $g$ which constructs the entire graph, $\Tilde{g}$ only needs to construct the graph for one causal/reasoning step. 

Sub-step 2 applies the causal function $f$ to infer 
\vspace{-1mm}
\begin{equation}
f(\Tilde{g}(s^0)) \overset{def}{=} \{f(p(v)) | \ p(v) \in \Tilde{g}(s^0)\}.
\end{equation}
Sub-step 3 puts the resulting one-step-forward subgraph into the unstructured data. 
Note that $s_1^0,\dots, s_{i_0}^0$ are vertices of $g(s^0)$, where $v \in s^0$ means $v$ is a vertex of $g(s^0)$ that is valued by some $s_j^0 \in s^0$. 
Since the vertices of $s^0$ and the one-step-forward vertices $f(\Tilde{g}(s^0))$ construct a subgraph of $g(s^0)$, and $g^{-1}$ is well-defined, we simply map the subgraph of vertices $(s^0, f(\Tilde{g}(s^0)))$ into the unstructured data by 
\begin{equation}
s^1 = g^{-1} (f(\Tilde{g}(s^0)), s^0).
\end{equation}
Sub-step 3 can be complex. For instance, there may exist input vertices that infer multiple causal vertices by $f$, i.e. $|f(p(v))| > 1$, then the question is how to represent $f(p(v))$ by the unstructured data. It's also a question of where to put $f(p(v))$ back in $s^0$ when $p(v)$ is not a successive sub-sequence of $s^0$. The other question is how to deal with the useless vertices $U^0 = \{v \in \bigcup_{v \in W^0} p(v)|\, t(v) \subseteq W^0 \}$. 

In this work, we define $s^1 = g^{-1} (f(\Tilde{g}(s^0)), s^0)$ by specific rules of the reasoning problem. 
But it's unclear if sub-step 3 can be learned in general. We leave this to our future work. 

In Sec. \ref{sec: dag}, we discussed the causal function, which is used in sub-step 2 of CoT. Now we discuss the possibility of learning sub-step 1. Denote 
\vspace{-1mm}
\begin{equation}
D \subseteq \{(\Tilde{g}(s), s)|\, s = g^{-1}(G)\}
\end{equation}
to be the dataset of $\Tilde{g}$. $G$ represents the underlying structure of the reasoning problem, and $s = g^{-1}(G)$ is the corresponding unstructured or sequence data of $G$. 
When $v_i = s_i$ and $v_j = s_j$, denote
\begin{equation}
d(v_i, v_j; s) \overset{def}{=} |i - j|,
\end{equation} 
which defines the distance between two vertices that were originally in $s^0$ to be their index distance in $s^0$. 

We now introduce the important notion of $R$, the \textit{maximal input element distance for a reasoning step}, which decides whether \textbf{LG} can be achieved. 
Define 
\begin{equation}
R \overset{def}{=} \sup_{s} \sup_{v \in s} \max_{v' \in (p \circ t)(v)} d(v', v; s). 
\label{eq: max_input_ele_dist}
\end{equation}
The $\sup$ is for any element $v$ in the unstructured sequence of the problem, the $\max_{v' \in (p \circ t)(v)}$ is the maximal distance between any two elements that should be involved in the next calculation/reasoning/causal step. 
$R$ is the maximal distance between the elements that should be calculated next in the same calculation step of the problem.

\vspace{+2mm}
\begin{theorem}
\label{thm: find_input_vertices_is_possible}
For $R < \infty$, if $D = X^{4R +1}$,
then there exists an approximation function $\Hat{g}: X^{4R + 1} \rightarrow \{0, 1\}$ 
s.t. $\Hat{g} (s')|_{s'_c} = \Tilde{g} (s)|_{s'_c}$, where $s' \subseteq s$ is a $4R+1$ sub-interval of $s$, and $s'_c$ is the central element of $s'$. 
\end{theorem}
The proof is given in \textit{Appendix}~\ref{sec.mied}. 

For any problem with $G$ as the underlying DAG of the problem ($G$ is not accessible), $\forall\, v \in g^{-1}(G),\,\forall \, v' \in (p \circ t)(v)$, $d(v', v)$ is uniformly bounded by a constant $R$. $R = \sup \max_{v' \in (p \circ t)(v)} d(v', v)$ is a primitive property of the reasoning DAG $G$ and function $g^{-1}$ transforms $G$ into the unstructured data.
For instance, if $s^0 = (a_0 + a_1) \times (a_2 - (a_3 / a_4))$, we have $\Tilde{g}(s^0) = \{(a_0 + a_1),\,(a_3 / a_4)\}$. It's easy to see $R = 4$ for arithmetic problems with `(' and `)'. 

Theorem \ref{thm: find_input_vertices_is_possible} requires $D = X^{4R + 1}$ because the combinations within a $(4R+1)$-length window {cover all cases that enable a correct order to be learned}. 
For instance, in `$\dots e \times a + b + c \times d \dots$', it's obvious that $R = 2$ because each operation or reasoning step involves 3 elements, e.g., $e \times a$, $a + b$, $b + c$, $c \times d$. To decide whether $b$ should be calculated first, we look at neighbors of $b$ in radius $R$, which is a window of length $2R + 1$. In this window $a + b + c$, we know $a + b$ should be calculated first. However, this is not enough because we also need to consider neighbors of $a$ in radius $R$ to validate whether $a + b$ should be calculated first. {In the window $e \times a + b$, we know $e \times a$ should be calculated first. Therefore, in this $(4R+1)$-length window $e \times a + b + c \times d$, we know $b$ in the center shouldn't be calculated first.} A $(4R+1)$-length window guarantees correctness because of the consistency, which is shown below in Theorem \ref{thm: (1,4R+1)-consist}.
$D = X^{M}$ ($M < 4R+1$) may not guarantee correctness. An example is given in \textit{Appendix} \ref{sec.mied}.

When $R < \infty$, Theorem \ref{thm: find_input_vertices_is_possible} holds, where a learned function $\Hat{g}$ predicts which elements in a $(4R+1)$-length window should be reasoned next. When $|\mathbf{X}| < \infty$, Theorem \ref{thm: full_causal_function} 
holds, where a learned causal function $\Hat{f}$ takes the predicted elements of $\Hat{g}$ as input to predict the value of the next/causal vertex. By Theorem \ref{thm: recursive_dag_solved}, the recursive process solves the problem of arbitrary length, i.e., \textbf{LG}.


\subsection{Dealing with $R = \infty$}
\label{sec: (n, r)-consistent}

We have shown that LG is solvable for problems with $R < \infty$. However, many reasoning problems have $R = \infty$. 
Our example above illustrated that for $R < \infty$, an interval of length $4R + 1$ is sufficient for the learner to learn to determine whether the central element (e.g., $b$ above) of the interval is an element in the next reasoning step or not. We can extend the idea to $R = \infty$ by considering the existence of $n$ intervals that can determine whether an element should be involved in the next reasoning step. A more general LG condition, called \textbf{(\textit{n}, \textit{r})-\textit{consistency}}, is proposed. The intuition is that for any set of $n$ $r$-length contiguous sub-sequences (or intervals) of elements when it appears in any \textbf{instance} (e.g., $23+345$) of a \textbf{problem class} (e.g., \textit{addition}), the central element of each of the $n$ sub-sequences must be either in the next reasoning step or not in the next reasoning step among all possible instances of the problem class. This ensures that there is no ambiguity in learning to find the elements involved in the next reasoning step and the resulting thus can be applied to any instance of the problem at test time. We also show that $R < \infty$ is $(1, 4R+1)$-consistent, i.e., $R < \infty$ is a special case of $(n,r)$-consistency. For simplicity, we use a `\textit{problem}' to refer to a `\textit{problem class}' below.
\begin{definition}
\label{def: (n,r)-consistent}
A reasoning problem is $(n, r)$\textit{-consistent} for integers $n$ and $r$, if, \textbf{(i)} for any set $S$ of $n$ $r$-length sequences of elements appearing in the CoT formulation of an instance of the problem as sub-sequences, whether or not the central element of each sub-sequence belongs to the set of elements to be calculated next is always \textit{the same} for any problem instance, and \textbf{(ii)} for any problem instance, there always exist $n$ $r$-length contiguous sub-sequences that cover all input elements of a reasoning step. Formally, 

\textbf{(i)} Let $I_j$ be a sequence in the $n$ $r$-length sequences, i.e., $I_j \in S$. For $\forall\, \{I_1, \dots, I_n |\, I_j = s_{j_1} \dots s_{j_r},\, j = 1,\dots,n \}$, 
$\forall\, s^0, s^1$ (CoT steps of any problem instances) that contains contiguous sub-sequences $\{I_1, \dots, I_n\}$, i.e. $\forall\, s^0, s^1 \in \{s | s = s_1 s_2 \dots s_{i_0},\,\exists\, i_j\ s.t.\ s_{i_j} \dots s_{i_j + r - 1} = I_j, j = 1, \dots, n\}$, we always have 
\begin{equation}
\Tilde{g}(s^0)|_{\{s_{j_{c}} |\, j = 1, \dots, n\}} = \Tilde{g}(s^1)|_{\{s_{j_{c}} |\, j = 1, \dots, n\}},
\label{eq: (n,r)-consistent}
\end{equation} 
where $s_{j_{c}} = s_{j_{\lfloor\frac{r + 1}{2}\rfloor}}$ is the central element of $I_j$, and $\Tilde{g}(s^i)|_{K} \overset{def}{=} \{\{v' \in p(v) \cap K\} | p(v) \in \Tilde{g}(s^i)\}$ 
are the elements for the next reasoning step in set $K$.
Note that to ensure that every element can be a central element of a $r$-length sub-sequence, we can add \textbf{empty elements} at the beginning and the end of a problem instance. 

\textbf{(ii)} Let $s^0 = s^0_1 \dots s^0_{i_0}$ be a CoT step of an instance of the problem. There exists $p(v_0) \in \Tilde{g}(s^0)$, $n$ $r$-length contiguous sub-sequences $I_j = s^0_{i_j} \dots s^0_{i_j + r - 1},\, j = 1, \dots, n$, s.t. $p(v_0) \subseteq \bigcup_{j=1}^n I_j$. 

\end{definition}
\vspace{+2mm}


$(n, r)$-consistency is an \textit{essential property} of a CoT formulation of a reasoning problem. It reflects the capacity to identify the elements to be calculated next from the finite contiguous sub-sequences (or \textbf{intervals}) and to apply the causal function on the input elements concurrently. Let's use the \textit{addition-[1-line]} problem (the \textit{addition} problem formulated in CoT in one line)  as an example (see its detailed CoT process in Table \ref{tab:examples} of \textit{Appendix}~\ref{appendix.examples}). We check whether the problem is $(3, 3)$-consistent or not. 

\textbf{(1)} For a problem instance's CoT step of `$123+567=~c0$', $\{\text{`}2\text{'}, \text{`}6\text{'}, \text{`}~c\text{'}\}$ (notice the empty element before $c$) are the elements to be calculated next ($?$ represents $0$ is carried and $c$ means $1$ is carried) and is covered by $\{I_1, I_2, I_3\} = \{\text{`}123\text{'}, \text{`}567\text{'}, \text{`}~c0\text{'}\}$. 
It's obvious that for any problem instance, there exists $3$ intervals of length $3$ covering all the elements to be calculated next. 
For $(3, 3)$, \textit{addition-[1-line]} satisfies Def~\ref{def: (n,r)-consistent} (ii). 
Note that Def~\ref{def: (n,r)-consistent} (ii) for $(n, r)$-consistent is not satisfied for $n < 3$ because $1$ or $2$ interval(s) cannot cover all the elements to be calculated next. 

\textbf{(2)} Let's check the central elements of $\{I_1, I_2, I_3\} = \{\text{`}123\text{'}, \text{`}567\text{'}, \text{`}~c0\text{'}\}$.
For the CoT step `$123+567=~c0$', the central elements of $\{I_1, I_2, I_3\}$ are elements to be calculated {\color{black}next.} 
However, for a CoT step of another problem instance $12342+45678=~c0$, $\{\text{`}4\text{'}, \text{`}7\text{'}, \text{`}~c\text{'}\}$ are the elements to be calculated next, but the central elements of $\{I_1, I_2\}$ (i.e., \{\text{`}123\text{'}, \text{`}567\text{'}\}) are not. So for $\{I_1, I_2, I_3\}$, these two CoT steps have different $\Tilde{g}(s^0)|_{\{s_{j_{c}} |\, j = 1, 2, 3\}}$, i.e., $(3, 3)$ doesn't satisfy Def~\ref{def: (n,r)-consistent} (i). 
Thus, \textit{addition-[1-line]} is not $(3, 3)$-consistent. 
Similarly, the \textit{addition-[1-line]} problem is not $(n,r)$-consistent ($n \geq 3$) because when '$?$' or `$c$' is the central element of a $r$-length sub-sequence (or interval), the central elements of the other $n-1$ sub-sequences may or may not be elements used in the next reasoning step for different problem instances. Thus, \textit{addition-[1-line]} is not $(n,r)$-consistent for any $n$ or $r$. 

However, we can formulate addition into a $(3, 3)$-consistent problem with 2 lines, i.e. \textit{addition-[2-line]} (see the CoT process in Table \ref{tab:examples} of \textit{Appendix}~\ref{appendix.examples}). The two CoT examples above become, 
$$
\begin{aligned}
123+567=~c0 &\Rightarrow \left(\begin{aligned}
&123 + 567 = ~~c0 \\
&\ \ I\ \ \ \ \ \  i\ \ \ \ \ \ \ \ \,J \\
\end{aligned}\right), \\
12342 + 45678 =~c0 &\Rightarrow \left(\begin{aligned}
&12342 + 45678 = ~~c0 \\
&~~~~~I\ \ \ \ \ \  \ \ \ i\ \ \ \ \ \ \ J \\
\end{aligned}\right), \\
\end{aligned}
$$
where $I$ and $i$ indicate the digits to be added next and $J$ indicates the position of the output. In this case, each element has 2 dimensions and when the second dimension is not $I$, $i$, or $J$, it is an empty token, which is also significant. $(3, 3)$ satisfies both (i) and (ii) of Def~\ref{def: (n,r)-consistent} due to adding the position indicators as the second dimension of each element.  Thus, \textit{addition-[2-line]} is $(3, 3)$-consistent. 

The $(n, r)$-consistency is not unique, which is shown in the following Theorem.  
\vspace{+2mm}
\begin{theorem}
If a problem is $(n, r)$-consistent, it's also $(n, r')$-consistent, $\forall\, r' \geq r$. 
\label{thm:  n,r_implies_n,r'}
\end{theorem}
The proof is given in \textit{Appendix}~\ref{sec.(n,r)-consistency}. When a problem satisfies $R < \infty$, Theorem \ref{thm: (1,4R+1)-consist} shows that it is $(1, 4R+1)$-consistent. Thus, $R < \infty$ is a special case of $(n,r)$-consistent. 
\vspace{+2mm}
\begin{theorem}
If i) $R < \infty$ (defined in Eq.~\eqref{eq: max_input_ele_dist}) and ii) each element belongs to at most one reasoning step, the problem is $(1, 4R+1)$-consistent. 
\label{thm: (1,4R+1)-consist}
\end{theorem}
The proof is given in \textit{Appendix}~\ref{sec.mied}. If a problem is $(n,r)$-consistent, for any $n$ $r$-length intervals $\{I_1, \dots, I_n\}$ of a problem instance, since $\Tilde{g}(s^0)|_{\{s_{j_{c}} |\, j = 1, \dots, n\}}$ is identical for $\forall\, s^0 \text{ containing } \{I_1, \dots, I_n\}$, to simplify our notation, we can define the following function $\gamma: X^{r \times n} \rightarrow 2^{n}$:
\begin{equation}
\gamma (I_1, \dots, I_n) \overset{def}{=} \Tilde{g}(s^0)|_{\{s_{j_{c}} |\, j = 1, \dots, n\}}.
\end{equation}
The $(n, r)$-consistent is the \textit{necessary} condition to well-define $\gamma$. 
With $\gamma$ defined, it's intuitive to scan $s^0$ by sliding $\{I_1, \dots, I_n\}$ to obtain $\Tilde{g}(s^0)$ 
as shown in Theorem \ref{thm: recover_g_by_gamma}. 
\vspace{+2mm}
\begin{theorem}
If the problem is $(n, r)$-consistent, then $\gamma$ is well-defined. For $\forall\, s^0$, the elements involved in the next reasoning step, i.e. $\bigcup \Tilde{g}(s^0)$, can be found by $\gamma$. 
\label{thm: recover_g_by_gamma}
\end{theorem}
See \textit{Appendix}~\ref{sec.(n,r)-consistency} for the proof. Now we know that when $\gamma$ is given, $\Tilde{g}(s^0)$ can be recovered by Theorem \ref{thm: recover_g_by_gamma} regardless of the length. The only question is whether $\gamma$ is learnable. Indeed, $\gamma: X^{r \times n} \rightarrow 2^n$ can be perfectly learned as $\gamma$ taking finite ($r \times n$) elements as input. See Theorem \ref{thm: learn_gamma_is_possible}. 
\vspace{+2mm}
\begin{theorem}
If a problem is $(n, r)$-consistent, then there exists an approximation function $\Hat{\gamma}: X^{r \times n} \rightarrow 2^{n}$, s.t. $\Hat{\gamma}(I_1, \dots, I_n) = \gamma(I_1, \dots, I_n),\,\forall\, \{I_1, \dots, I_n\} \in X^{r \times n}$. 
\label{thm: learn_gamma_is_possible}
\end{theorem}
\vspace{-2mm}
See \textit{Appendix}~\ref{sec.(n,r)-consistency} for the proof. When the problem is $(n, r)$-consistent, $\gamma$ is well-defined. Since $\gamma$ only takes finite elements as input, it can be perfectly learned (Theorem \ref{thm: learn_gamma_is_possible}). When $\gamma$ is given, for $\forall\, s^0$ of arbitrary  length, $\Tilde{g}(s^0)$ can be recovered by $\gamma$ based on Theorem \ref{thm: recover_g_by_gamma}. When $\Tilde{g}(s^0)$ is given, and the causal function is well-learned by Theorem \ref{thm: full_causal_function}, the problem can be recursively solved based on Theorem \ref{thm: recursive_dag_solved}.

To close, we summarize the proposed theory. In
Sec.~\ref{sec: dag}, we assume that the DAG structure of a reasoning problem is given and show that the problem is recursively solvable if the input space of the causal function is finite. 
In Sec.~\ref{sec:cot_and_others}, we assume that only unstructured sequence data is given, which is the practical case, and show that if the problem has the property of $R < \infty$, it is solvable for LG. 
This subsection extends the result of Sec. \ref{sec:cot_and_others} and shows that a problem with $R = \infty$ and $(n, r)$-consistency is also solvable for LG. 



\section{Related Work}
\label{sec.related}
{We now review the related work. We focus on the theoretical work here and leave the empirical work to \textit{Appendix}~\ref{appendix.related}, which includes: 1) evaluations or improvements on the reasoning capabilities of LLMs and CoT and its variants, 2) modifying the Transformer and/or learning biases to better solve LG, but they still cannot solve it~\cite{duan2023interpolation,zhou2023algorithms,jelassi2023length,chi2023transformer,nangia2018listops,bowman2015tree,tay2021long,chowdhury2023monotonic} and 3) LG for text generation (a different problem)~\cite{sun2022length,press2021train,ruoss2023randomized,han2023lm}. 

\cite{abbe2023generalization} studied out-of-distribution (OOD) generalization of reasoning where part of the domain is entirely unseen during training, e.g., some value combinations are missing. It analyzes the learning biases of different network architectures and activation functions. All resulting models may  
predict wrongly on the OOD data.  {They also analyzed LG based on biases and used curriculum learning to improve the performance of the parity problem. Our work identifies and proves conditions under which the LG can be solved.} 

LG is related to OOD generalization.~However, there is a key difference. Since the maximal size/length of the training problems is always finite, a larger size/length problem can always appear at test time, which can be seen as OOD. But such an OOD is unavoidable regardless of how much data is used in training as long as it is finite. 
In \cite{abbe2023generalization}, OOD means some values or value combinations never appeared in training but appeared in testing.~This problem is solvable 
with more diverse training data. 
} 

\cite{wies2023sub} proved that when sufficient intermediate steps (or CoT) are available, a neural network can efficiently learn any function in the \textbf{P} time complexity class. In addition, 
there exist functions in the \textbf{P} time complexity class that cannot be learned by any polynomial time learning algorithm without CoT. \cite{feng2023towards} showed why CoT works on problems that can be decomposed into sub-problems (DAG in our case). They also proved that it is not learnable directly without CoT. \cite{li2023dissecting} showed that CoT can enable the model to identify each step and then work on the step before moving to the next step in the CoT chain. \cite{prystawski2023think} studied why and how CoT works in LLMs. 
\cite{malach2023auto} proved that even simple models like linear next-token predictors trained on CoT data are universal learners. The paper also introduces the \textit{length complexity} to measure how many intermediate tokens are required to learn a function.

However, the theorems in these papers 
are based on the traditional i.i.d setting, i.e., under the given length/size $N$. Their statements are like ``for $N > 0$, given a dataset with training problems no longer than $N$, for any problem with length $N' \leq N$, it can be solved under some PAC upper bound.'' The key limitation is that the training problem length and testing problem length are the same with no LG. 
Our theory doesn't have this limitation and covers LG. Our statement is ``for $N > 0$, given a dataset with problems no longer than $N$, for any problem with arbitrary length $N'$, it can be solved if $|\mathbf{X}| < \infty$ and $(n,r)$-consistent.''

\section{Experiments}
\label{sec:experiment}

Given a reasoning problem as an unstructured sequence (the realistic scenario), our experiments verify three key aspects 
\begin{enumerate}
    \item 
    If the input space of the causal function of the CoT formulation is finite, i.e., $|\mathbf{X}| < \infty$, and $R < \infty$, the problem can be solved with LG. 
    \item For a CoT formulation of a problem with $R=\infty$, if it is $(n,r)$-consistent, it is also solvable for LG.
    \item {For the same reasoning problem, one CoT formulation may not be solvable for LG, but another may}.
\end{enumerate}

\begin{table*}[h]
\vspace{-3mm}
\caption{Experimental settings. Train Length: Training length. LG Test $i$: Length generalization test with longer lengths.}
\vspace{2mm}
\centering
\scalebox{0.85}{
\begin{tabular}{c|cccccc}
\toprule
 & Train Length & LG Test 1 & LG Test 2 & LG Test 3 & LG Test 4 & LG Test 5 \\
\midrule
arctan & $r \in (1/2, 2)$ & $r \in (1/3, 3)$ & $r \in (1/4, 4)$ & $r \in (1/5, 5)$ & $r \in (1/6, 6)$ & $r \in (1/10, 10)$ \\
arithmetic in $F_7$  & $L \in [3, 20)$ & $L \in [3, 30)$ & $L \in [3, 40)$ & $L \in [3, 50)$ & $L \in [3, 60)$ & $L \in [3, 100)$ \\
parity-[2-line]  & $L \in [1, 8)$ & $L \in [1, 9)$ & $L \in [1, 10)$ & $L \in [1, 11)$ & $L \in [1, 16)$ & $L \in [1, 21)$ \\
addition-[1/2/3-line] & $a, b \in [0, 10^8)$ & $a, b \in [0, 10^9)$ & $a, b \in [0, 10^{10})$ & $a, b \in [0, 10^{11})$ & $a, b \in [0, 10^{16})$ & $a, b \in [0, 10^{21})$ \\
multiplication-[1/8-line] & $a, b \in [0, 10^6)$ & $a, b \in [0, 10^7)$ & $a, b \in [0, 10^8)$ & $a, b \in [0, 10^9)$ & $a, b \in [0, 10^{10})$ & $a, b \in [0, 10^{11})$ \\
\bottomrule
\end{tabular}
}
\label{tab:experiment_setting}
\end{table*}

\subsection{Experimental Problems}
We use 5 reasoning problems: (1) \textbf{arctan}, (2) \textbf{arithmetic in $F_7$} (the finite prime field with seven elements), (3) \textbf{parity}, (4) \textbf{addition} (with 3 different CoT formulations), and (5) \textbf{multiplication} (with 2 different CoT formulations).  
We list the detailed parameters of the training and testing datasets of each problem in Table \ref{tab:experiment_setting} (explained below). The properties of each problem are listed in Table \ref{tab:experiment_property}. 
Some training examples for each problem are given in Table \ref{tab:examples} in \textit{Appendix}~\ref{appendix.examples}, which also describes the \textbf{implementation details}. 

\begin{table}
\caption{Properties of the problems. `-cs' means `-consistency'.} 
\vspace{2 mm}
\centering
\scalebox{0.85}{
\begin{tabular}{c|c|c|c}
\toprule
 & $|\mathbf{X}|$ in & $R$ in & $(n, r)$-cs \\
  & Thm \ref{thm: full_causal_function} & Thm \ref{thm: find_input_vertices_is_possible} & (not unique) \\
 
\midrule
arctan & $|\mathbf{X}|=\infty$ & $R = 1$ & $(1, 2)$-cs\\
arithmetic in $F_7$ & $|\mathbf{X}| = 13^5$ & $R=4$ & $(1, 17)$-cs\\
parity-[2-line] & $|\mathbf{X}| = 3^6$ & $R=1$ & $(1, 3)$-cs \\
addition-[1-line] & $|\mathbf{X}| = 14^4$ & $R=\infty$ & not $(n, r)$-cs \\
addition-[2-line] & $|\mathbf{X}| = 17^{18}$ & $R=\infty$ & $(3, 3)$-cs \\
addition-[3-line] & $|\mathbf{X}| = 14^9$ & $R=1$ & $(1, 3)$-cs \\
multiplication-[1-line] & $|\mathbf{X}| = \infty$ & $R=\infty$  & not $(n, r)$-cs \\
multiplication-[8-line] & $|\mathbf{X}| = 19^{216}$ & $R=\infty$ & $(9, 3)$-cs \\
\bottomrule
\end{tabular}
}
\label{tab:experiment_property}
\end{table}
\subsection{Results and Analysis} 
The \textbf{arctan} problem verifies that the causal function may make mistakes when \textit{the input space is not finite}. The training data is sampled from an annulus of radius $r \in (1/2, 2)$. We test the performance on different annuluses, listed in Table \ref{tab:experiment_setting}. The accuracy is reported in Fig. \ref{fig:experiments}. It decays as the annulus becomes larger, which satisfies Corollary \ref{thm: not_full_causal_function_inf_x}. 

\begin{figure}
\caption{
Test results in accuracy.
} 
\vspace{+2mm}
\includegraphics[width=0.80\linewidth]{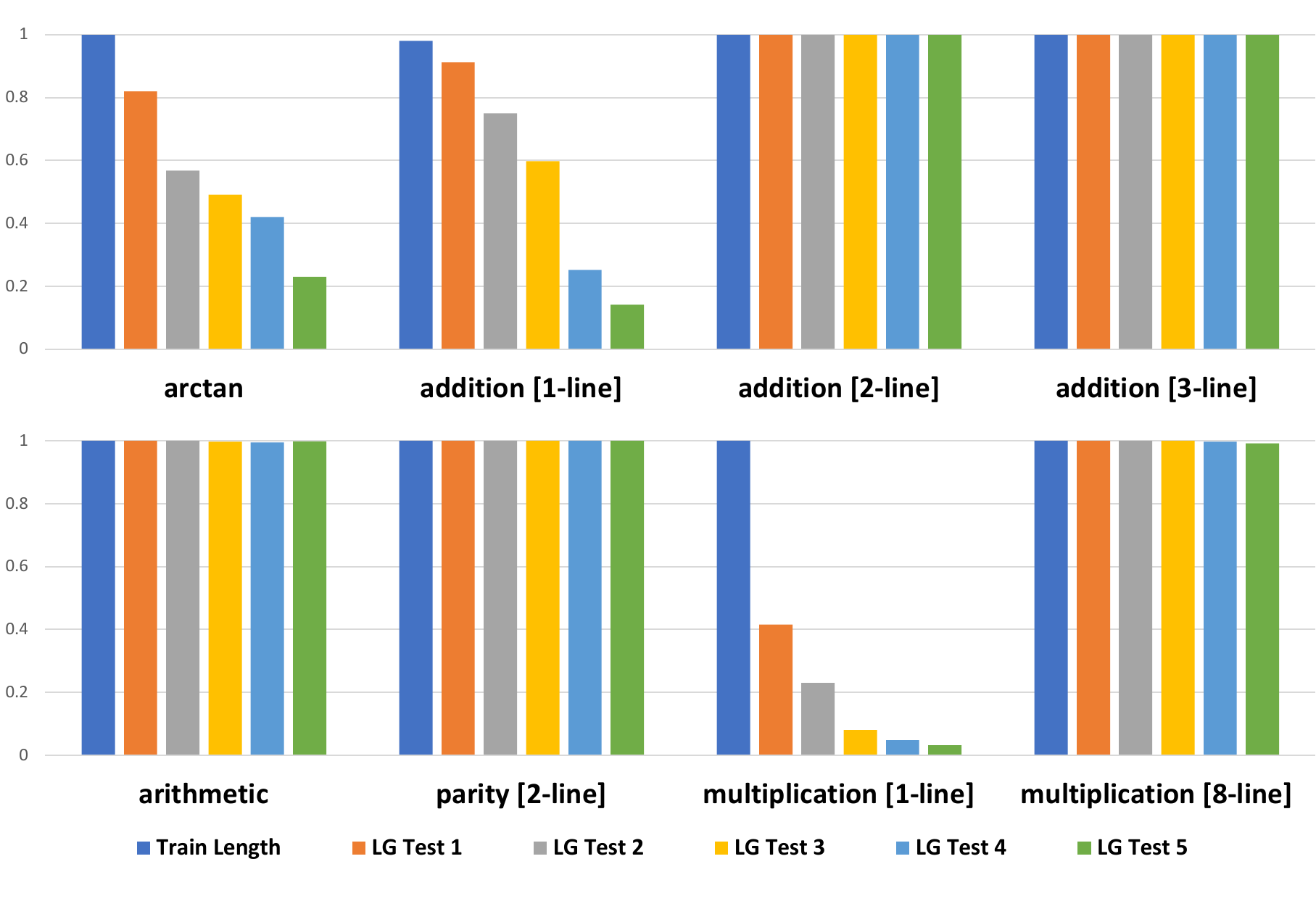}
\label{fig:experiments}
\end{figure}

The \textbf{arithmetic in $F_7$} problem is formulated in the usual way, e.g., $(3 + 2) \times 1 = 5 \times 1 = 5$. The input space is finite because
$|\mathbf{X}| = |\{(,),+,-,\times, /,0,1,2,3,4,5,6\}|^5 < \infty$.  
$R = 4$ as any combination of elements to be calculated in the next step is within a window of at most 5 elements, e.g., $(2+3)$. 
The training and testing settings are listed in Table \ref{tab:experiment_setting}. The training data has at most 20 elements, i.e. $L \in [3, 20)$, and the test data has at most 20, 30, 40, 50, 60, 100 elements. Fig. \ref{fig:experiments} shows this problem achieves 100\% accuracy. 

The \textbf{parity} problem is to decide whether there are an \textit{even} or \textit{odd} number of 1's in a sequence of $0$'s and $1$'s. We formulate it with 2 lines, i.e., \textit{parity-[2-line]}. On the $2$nd line, $?$ indicates the current position of the CoT process, $1$ represents \textit{odd} and $0$ represents \textit{even} (see Table \ref{tab:examples} in \textit{Appendix}~\ref{appendix.examples}). The input space $|\mathbf{X}| = |\{0, 1, ?\}|^6 < \infty$ as the causal function for each step only needs an interval of length $3$, which contains $3 \times 2$ elements. At each CoT step, only `$?$' and the element after `$?$' are input elements, by definition, $R = 1$. The problem is $(1, 3)$-consistent because (1) the central element of a $3$-length interval is the input element \textit{iff} `$?$' is on the left/center of the interval, and (2) an interval of length $3$ can cover `$?$' and the element after `$?$'.
Fig. \ref{fig:experiments} shows it achieves 100\% accuracy.

The \textbf{addition} problem is formulated in CoT in three ways: \textit{addition-[1-line]}, \textit{addition-[2-line]}, and \textit{addition-[3-line]}. The addition-[1-line] formulation does not achieve LG, but \textit{addition-[2-line]} and addition-[3-line] do. They are all trained with 7 or fewer digits additions, i.e., $a, b \in [0, 10^8)$ (Table~\ref{tab:experiment_setting}), and tested on fewer than or equal to 7, 8, 9, 10, 15, 20 digits additions.

Fig. \ref{fig:experiments} shows that \textit{addition-[1-line]} fails to generalize beyond 7 digits as Sec. \ref{sec: (n, r)-consistent} showed that it is not $(n, r)$-consistent. It also has $R=\infty$. Let us see why by considering $x=\text{`}285+9805=~?\text{'}$ ($?$ means 0 is carried).
The elements/digits to be calculated next are $x[2]=\text{`}5\text{'}$, $x[7]=\text{`}5\text{'}$ and $x[9]=\text{`}?\text{'}$. The maximal distance between them is $7$. This distance increases as the number of digits increases. By definition, $R = \sup \max_{v' \in (p \circ t)(v)} d(v', v) = \infty$. 

The \textit{addition-[2-line]}, which has been described in Sec. \ref{sec: (n, r)-consistent}, achieves 100\% accuracy (Fig.~\ref{fig:experiments}) because it's $(3, 3)$-consistent. Like \textit{addition-[1-line]}, it has $R = \infty$. $|\mathbf{X}|=|\{I, i, J, +, =, ?, c, 0, 1, 2, 3, 4, 5, 6, 7, 8, 9\}|^{18} < \infty$, because $3$ intervals of length $3$ have $3 \times 3 \times 2$ elements. 

{\color{black}For \textit{addition-[3-line]}, $x$ becomes (each element has 3 dimensions), 
\begin{equation}
x = \left(\begin{aligned}
\ 285 \\
9805 \\
\ \ \ ?
\end{aligned}\right) = (\text{`}\ 285\text{'}, \text{`}9805\text{'}, \text{`}\ \ \ ?\text{'})^T.
\label{eq.multi-d}
\end{equation}
In each step, $x[i]$ only needs to consider its right neighbor $x[i+1]$. 
$x[3] = \left(\text{`}5\text{'}, \text{`}5\text{'}, \text{`}?\text{'}\right)^T$ is enough for calculation.
$x[2] = \left(\text{`}8\text{'}, \text{`}0\text{'}, \text{`}\ \text{'}\right)^T$ only needs to consider $x[3]$.~The maximal distance of elements to be calculated next is always $1$ (i.e., $R = 1$), which doesn't depend on the number of digits. 
The problem is $(1, 3)$-consistent because (1) the central element of an interval of length $3$ is the input element \textit{iff} `?' or `c' is at the center/right of the interval, and (2) a $3$-length interval centering at `?' or `c' can cover all the input elements of one reasoning step. $|\mathbf{X}| = |\{+, =, ?, c, 0, 1, 2, 3, 4, 5, 6, 7, 8, 9\}|^9 < \infty$, because an interval of length $3$ has $3 \times 3$ elements.

Fig. \ref{fig:experiments} shows that the \textit{addition-[3-line]} formulation performs with $100\%$ accuracy on 7, 8, 9, 10, 15, and 20 digits because its $R < \infty$.}

The \textbf{multiplication} of two integers is formulated in two ways: \textit{multiplication-[1-line]} and \textit{multiplication-[8-line]}. 

For \textit{multiplication-[1-line]}, we decompose the problem into two stages. In the first stage, we transform multiplication into a summation of multiple integers. In the second stage, we solve the summation recursively. An example is shown in Table \ref{tab:examples} in \textit{Appendix}~\ref{appendix.examples}. For the second stage, it's solvable by \textit{addition-[2/3-line]}. However, we have $R = \infty$ in stage 1. For instance, let input[k] = $\text{`}a \times b=\underbrace{a+\dots+a}_{k}+?\text{'}$. When $k < b - 1$, output[k] = $\text{`}a \times b=\underbrace{a+\dots+a}_{k+1}+?\text{'}$. when $k = b - 1$, output[k] = $\text{`}a \times b=\underbrace{a+\dots+a}_{k+1}\text{'}$. In this example, whether to add `$+?$' or to go to the second stage depends on $b$ and the number of existing $a$'s. All the elements are input elements in a reasoning step. Therefore, $|\mathbf{X}| \geq \lim_{a, b \rightarrow \infty} |\mathbf{X}|^b = \infty$, the maximal distance between input elements $R$ is $\infty$. 
Since all the elements are input elements for the next reasoning step, for any $(n, r)$, there always exists a longer problem instance such that $n$ $r$-length intervals cannot cover all the input elements of a reasoning step. Thus it's not $(n, r)$-consistent. Since it has $|\mathbf{X}| =\infty$, $R = \infty$, and not $(n, r)$-consistent, it isn't solvable for LG. Fig.~\ref{fig:experiments} shows poor LG accuracy for \textit{multiplication-[1-line]}.

For the \textit{multiplication-[8-line]} formulation (each element has 8-dimensions), we present an example in Table \ref{tab:examples} in \textit{Appendix}~\ref{appendix.examples}. When calculating $a \times b$, since \textit{addition-[2/3-line]} solves LG, we only need to multiply each digit of $a$ and each digit of $b$. We use some position indicator tokens: Tokens `I' and `i' indicate the positions of the digits in $a$ and $b$ to multiply next, where any interval of length $3$ that contains `I' or `i' determines the digits to multiply next. Static tokens `E', `S', `e', and `s' are used to indicate the start and end of $a$ and $b$, where any interval of length $3$ that contains `E', `S', `e' or `s' determines the next position of `I' or `i'. Token `J' indicates the position of the multiplication result of the current step, where any interval of length $3$ that contains `J' determines the position of the result. Tokens `F' and `T' are determined by shifting `e' and `s' to the left by I'th tokens, which helps find the next position of `J'. More descriptions are given in \textit{Appendix}~\ref{appendix.examples}. The \textit{addition-[3-line]} method is applied to add the output of each CoT step to the answer. It is obvious that $R = \infty$ as the distance between the most left digit of $a$ and the most right digit of $b$ is arbitrarily large as $a, b \rightarrow \infty$. 
The problem is $(9, 3)$-consistent because (1) the central element of a $3$-length interval is an input element \textit{iff} the central element has one of these $9$ tokens or the right element has one of `I', `i', `F', `T', `J', and (2) the $9$ intervals of length $3$ centering at these $9$ tokens respectively can cover all the input elements of one reasoning step.
$|\mathbf{X}| = |\{E, S, e, s, I, i, F, T, J, 0, 1, 2, 3, 4, 5, 6, 7, 8, 9\}|^{216} < \infty$, because $9$ intervals of length $3$ with $8$-lines have $9 \times 3 \times 8$ elements. It is solvable for LG. Our experiments show the problem is solved with 100\% accuracy (see Fig. \ref{fig:experiments}).


\vspace{-1mm}

\section{Conclusion}
\label{sec: discussions}
Length generalization (LG) is a challenging problem in learning reasoning skills. 
There is little theoretical understanding so far. This paper identified and proved some sufficient conditions for LG. The theory is verified by learning to solve several challenging reasoning problems. 

\textit{Future directions}: (1) This work considers reasoning problems that can be structured as DAGs. However, there are reasoning problems that cannot be represented as DAGs (e.g., temporal reasoning). Studying reasoning problems that cannot be model as DAGs is an interesting future direction. (2) The proposed conditions are sufficient conditions. We still don't know the necessary conditions. Investigating necessary conditions is also an interesting future direction.  

\section*{Acknowledgments}
Bing Liu's work was supported in part by NSF grants (IIS-1910424, IIS-1838770, and CNS-2225427).
\nocite{langley00}

\bibliographystyle{alpha}
\bibliography{IAES2024}

\newcommand{\etalchar}[1]{$^{#1}$}
\begin{thebibliography}{WWS{\etalchar{+}}22b}

\bibitem[ABLR23]{abbe2023generalization}
Emmanuel Abbe, Samy Bengio, Aryo Lotfi, and Kevin Rizk.
\newblock Generalization on the unseen, logic reasoning and degree curriculum.
\newblock {\em arXiv preprint arXiv:2301.13105}, 2023.

\bibitem[AMA{\etalchar{+}}23]{ando2023evaluating}
Risako Ando, Takanobu Morishita, Hirohiko Abe, Koji Mineshima, and Mitsuhiro Okada.
\newblock Evaluating large language models with neubaroco: Syllogistic reasoning ability and human-like biases.
\newblock {\em arXiv preprint arXiv:2306.12567}, 2023.

\bibitem[AWA{\etalchar{+}}22]{anil2022exploring}
Cem Anil, Yuhuai Wu, Anders Andreassen, Aitor Lewkowycz, Vedant Misra, Vinay Ramasesh, Ambrose Slone, Guy Gur-Ari, Ethan Dyer, and Behnam Neyshabur.
\newblock Exploring length generalization in large language models.
\newblock {\em Advances in Neural Information Processing Systems}, 35:38546--38556, 2022.

\bibitem[BMP15]{bowman2015tree}
Samuel~R Bowman, Christopher~D Manning, and Christopher Potts.
\newblock Tree-structured composition in neural networks without tree-structured architectures.
\newblock {\em arXiv preprint arXiv:1506.04834}, 2015.

\bibitem[BMR{\etalchar{+}}20]{brown2020language}
Tom Brown, Benjamin Mann, Nick Ryder, Melanie Subbiah, Jared~D Kaplan, Prafulla Dhariwal, Arvind Neelakantan, Pranav Shyam, Girish Sastry, Amanda Askell, et~al.
\newblock Language models are few-shot learners.
\newblock {\em Advances in neural information processing systems}, 33:1877--1901, 2020.

\bibitem[BZJ{\etalchar{+}}23]{bi2023program}
Zhen Bi, Ningyu Zhang, Yinuo Jiang, Shumin Deng, Guozhou Zheng, and Huajun Chen.
\newblock When do program-of-thoughts work for reasoning?
\newblock {\em arXiv preprint arXiv:2308.15452}, 2023.

\bibitem[CC23a]{chowdhury2023monotonic}
Jishnu~Ray Chowdhury and Cornelia Caragea.
\newblock Monotonic location attention for length generalization.
\newblock {\em Proceedings of the 40th International Conference on Machine Learning (ICML-2023)}, 2023.

\bibitem[CC23b]{chowdhury2023recursion}
Jishnu~Ray Chowdhury and Cornelia Caragea.
\newblock Recursion in recursion: Two-level nested recursion for length generalization with scalability.
\newblock {\em Proceedings of 37th Conference on Neural Information Processing Systems (NeurIPS 2023)}, 2023.

\bibitem[CFRR23]{chi2023transformer}
Ta-Chung Chi, Ting-Han Fan, Alexander~I Rudnicky, and Peter~J Ramadge.
\newblock Transformer working memory enables regular language reasoning and natural language length extrapolation.
\newblock {\em arXiv preprint arXiv:2305.03796}, 2023.

\bibitem[CHL{\etalchar{+}}22]{chung2022scaling}
Hyung~Won Chung, Le~Hou, Shayne Longpre, Barret Zoph, Yi~Tay, William Fedus, Eric Li, Xuezhi Wang, Mostafa Dehghani, Siddhartha Brahma, et~al.
\newblock Scaling instruction-finetuned language models.
\newblock {\em arXiv preprint arXiv:2210.11416}, 2022.

\bibitem[CHLL21]{cao2021bottom}
Yixuan Cao, Feng Hong, Hongwei Li, and Ping Luo.
\newblock A bottom-up dag structure extraction model for math word problems.
\newblock In {\em Proceedings of the AAAI conference on artificial intelligence}, volume~35, pages 39--46, 2021.

\bibitem[CMWC22]{chen2022program}
Wenhu Chen, Xueguang Ma, Xinyi Wang, and William~W Cohen.
\newblock Program of thoughts prompting: Disentangling computation from reasoning for numerical reasoning tasks.
\newblock {\em arXiv preprint arXiv:2211.12588}, 2022.

\bibitem[CSC{\etalchar{+}}23]{chen2023measuring}
Yangyi Chen, Karan Sikka, Michael Cogswell, Heng Ji, and Ajay Divakaran.
\newblock Measuring and improving chain-of-thought reasoning in vision-language models.
\newblock {\em arXiv preprint arXiv:2309.04461}, 2023.

\bibitem[CXS{\etalchar{+}}23]{cheng2023binding}
Zhoujun Cheng, Tianbao Xie, Peng Shi, Chengzu Li, Rahul Nadkarni, Yushi Hu, Caiming Xiong, Dragomir Radev, Mari Ostendorf, Luke Zettlemoyer, et~al.
\newblock Binding language models in symbolic languages.
\newblock {\em ICLR-2023}, 2023.

\bibitem[DCLT18]{devlin2018bert}
Jacob Devlin, Ming-Wei Chang, Kenton Lee, and Kristina Toutanova.
\newblock Bert: Pre-training of deep bidirectional transformers for language understanding.
\newblock {\em arXiv preprint arXiv:1810.04805}, 2018.

\bibitem[DLS{\etalchar{+}}23]{dziri2023faith}
Nouha Dziri, Ximing Lu, Melanie Sclar, Xiang~Lorraine Li, Liwei Jian, Bill~Yuchen Lin, Peter West, Chandra Bhagavatula, Ronan~Le Bras, Jena~D Hwang, et~al.
\newblock Faith and fate: Limits of transformers on compositionality.
\newblock {\em arXiv preprint arXiv:2305.18654}, 2023.

\bibitem[DS23]{duan2023interpolation}
Shaoxiong Duan and Yining Shi.
\newblock From interpolation to extrapolation: Complete length generalization for arithmetic transformers.
\newblock {\em arXiv preprint arXiv:2310.11984}, 2023.

\bibitem[FGZ{\etalchar{+}}23]{feng2023towards}
Guhao Feng, Yuntian Gu, Bohang Zhang, Haotian Ye, Di~He, and Liwei Wang.
\newblock Towards revealing the mystery behind chain of thought: a theoretical perspective.
\newblock {\em arXiv preprint arXiv:2305.15408}, 2023.

\bibitem[FOC{\etalchar{+}}23]{fu2023chain}
Yao Fu, Litu Ou, Mingyu Chen, Yuhao Wan, Hao Peng, and Tushar Khot.
\newblock Chain-of-thought hub: A continuous effort to measure large language models' reasoning performance.
\newblock {\em arXiv preprint arXiv:2305.17306}, 2023.

\bibitem[GBWD23]{gendron2023large}
Ga{\"e}l Gendron, Qiming Bao, Michael Witbrock, and Gillian Dobbie.
\newblock Large language models are not abstract reasoners.
\newblock {\em arXiv preprint arXiv:2305.19555}, 2023.

\bibitem[GMZ{\etalchar{+}}23]{gao2023pal}
Luyu Gao, Aman Madaan, Shuyan Zhou, Uri Alon, Pengfei Liu, Yiming Yang, Jamie Callan, and Graham Neubig.
\newblock Pal: Program-aided language models.
\newblock In {\em International Conference on Machine Learning}, pages 10764--10799. PMLR, 2023.

\bibitem[Has95]{hassoun1995fundamentals}
Mohamad~H Hassoun.
\newblock {\em Fundamentals of artificial neural networks}.
\newblock MIT press, 1995.

\bibitem[Hay98]{haykin1998neural}
Simon Haykin.
\newblock {\em Neural networks: a comprehensive foundation}.
\newblock Prentice Hall PTR, 1998.

\bibitem[HLY{\etalchar{+}}23]{hsieh2023distilling}
Cheng-Yu Hsieh, Chun-Liang Li, Chih-Kuan Yeh, Hootan Nakhost, Yasuhisa Fujii, Alexander Ratner, Ranjay Krishna, Chen-Yu Lee, and Tomas Pfister.
\newblock Distilling step-by-step! outperforming larger language models with less training data and smaller model sizes.
\newblock {\em Findings of the Association for Computational Linguistics (ACL2023)}, 2023.

\bibitem[HQL{\etalchar{+}}23]{hu2023tree}
Pengbo Hu, Ji~Qi, Xingyu Li, Hong Li, Xinqi Wang, Bing Quan, Ruiyu Wang, and Yi~Zhou.
\newblock Tree-of-mixed-thought: Combining fast and slow thinking for multi-hop visual reasoning.
\newblock {\em arXiv preprint arXiv:2308.09658}, 2023.

\bibitem[HWX{\etalchar{+}}23]{han2023lm}
Chi Han, Qifan Wang, Wenhan Xiong, Yu~Chen, Heng Ji, and Sinong Wang.
\newblock Lm-infinite: Simple on-the-fly length generalization for large language models.
\newblock {\em arXiv preprint arXiv:2308.16137}, 2023.

\bibitem[HYLZ23]{hu2023code}
Yi~Hu, Haotong Yang, Zhouchen Lin, and Muhan Zhang.
\newblock Code prompting: a neural symbolic method for complex reasoning in large language models.
\newblock {\em arXiv preprint arXiv:2305.18507}, 2023.

\bibitem[HZL{\etalchar{+}}22]{huang2022directed}
Fei Huang, Hao Zhou, Yang Liu, Hang Li, and Minlie Huang.
\newblock Directed acyclic transformer for non-autoregressive machine translation.
\newblock In {\em International Conference on Machine Learning}, pages 9410--9428, 2022.

\bibitem[JdDE{\etalchar{+}}23]{jelassi2023length}
Samy Jelassi, St{\'e}phane d'Ascoli, Carles Domingo-Enrich, Yuhuai Wu, Yuanzhi Li, and Fran{\c{c}}ois Charton.
\newblock Length generalization in arithmetic transformers.
\newblock {\em arXiv preprint arXiv:2306.15400}, 2023.

\bibitem[KKHH23]{kim-etal-2023-athena}
Jb. Kim, Hazel Kim, Joonghyuk Hahn, and Yo-Sub Han.
\newblock {ATHENA}: Mathematical reasoning with thought expansion.
\newblock In Houda Bouamor, Juan Pino, and Kalika Bali, editors, {\em Proceedings of the 2023 Conference on Empirical Methods in Natural Language Processing}, pages 16315--16327, Singapore, December 2023. Association for Computational Linguistics.

\bibitem[KPR{\etalchar{+}}23]{kazemnejad2023impact}
Amirhossein Kazemnejad, Inkit Padhi, Karthikeyan~Natesan Ramamurthy, Payel Das, and Siva Reddy.
\newblock The impact of positional encoding on length generalization in transformers.
\newblock {\em 37th Conference on Neural Information Processing Systems (NeurIPS 2023)}, 2023.

\bibitem[Lan00]{langley00}
P.~Langley.
\newblock Crafting papers on machine learning.
\newblock In Pat Langley, editor, {\em Proceedings of the 17th International Conference on Machine Learning (ICML 2000)}, pages 1207--1216, Stanford, CA, 2000. Morgan Kaufmann.

\bibitem[LFL{\etalchar{+}}23]{ling2023deductive}
Zhan Ling, Yunhao Fang, Xuanlin Li, Zhiao Huang, Mingu Lee, Roland Memisevic, and Hao Su.
\newblock Deductive verification of chain-of-thought reasoning.
\newblock {\em arXiv preprint arXiv:2306.03872}, 2023.

\bibitem[LHS{\etalchar{+}}23]{lyu2023faithful}
Qing Lyu, Shreya Havaldar, Adam Stein, Li~Zhang, Delip Rao, Eric Wong, Marianna Apidianaki, and Chris Callison-Burch.
\newblock Faithful chain-of-thought reasoning.
\newblock {\em arXiv preprint arXiv:2301.13379}, 2023.

\bibitem[LK23]{lee2023recursion}
Soochan Lee and Gunhee Kim.
\newblock Recursion of thought: A divide-and-conquer approach to multi-context reasoning with language models.
\newblock {\em arXiv preprint arXiv:2306.06891}, 2023.

\bibitem[LL23]{liu2023goat}
Tiedong Liu and Bryan Kian~Hsiang Low.
\newblock Goat: Fine-tuned llama outperforms gpt-4 on arithmetic tasks.
\newblock {\em arXiv preprint arXiv:2305.14201}, 2023.

\bibitem[LNT{\etalchar{+}}23]{liu2023evaluating}
Hanmeng Liu, Ruoxi Ning, Zhiyang Teng, Jian Liu, Qiji Zhou, and Yue Zhang.
\newblock Evaluating the logical reasoning ability of chatgpt and gpt-4.
\newblock {\em arXiv preprint arXiv:2304.03439}, 2023.

\bibitem[Lon23]{long2023large}
Jieyi Long.
\newblock Large language model guided tree-of-thought.
\newblock {\em arXiv preprint arXiv:2305.08291}, 2023.

\bibitem[LSG{\etalchar{+}}23]{li2023dissecting}
Yingcong Li, Kartik Sreenivasan, Angeliki Giannou, Dimitris Papailiopoulos, and Samet Oymak.
\newblock Dissecting chain-of-thought: A study on compositional in-context learning of mlps.
\newblock {\em arXiv preprint arXiv:2305.18869}, 2023.

\bibitem[LSL{\etalchar{+}}23]{lee2023teaching}
Nayoung Lee, Kartik Sreenivasan, Jason~D Lee, Kangwook Lee, and Dimitris Papailiopoulos.
\newblock Teaching arithmetic to small transformers.
\newblock {\em arXiv preprint arXiv:2307.03381}, 2023.

\bibitem[LYE23]{li2023counterfactual}
Jiaxuan Li, Lang Yu, and Allyson Ettinger.
\newblock Counterfactual reasoning: Testing language models' understanding of hypothetical scenarios.
\newblock {\em arXiv preprint arXiv:2305.16572}, 2023.

\bibitem[Mal23]{malach2023auto}
Eran Malach.
\newblock Auto-regressive next-token predictors are universal learners.
\newblock {\em arXiv preprint arXiv:2309.06979}, 2023.

\bibitem[MMJ{\etalchar{+}}23]{mukherjee2023orca}
Subhabrata Mukherjee, Arindam Mitra, Ganesh Jawahar, Sahaj Agarwal, Hamid Palangi, and Ahmed Awadallah.
\newblock Orca: Progressive learning from complex explanation traces of gpt-4.
\newblock {\em arXiv preprint arXiv:2306.02707}, 2023.

\bibitem[MVTF23]{meadows2023symbolic}
Jordan Meadows, Marco Valentino, Damien Teney, and Andre Freitas.
\newblock A symbolic framework for systematic evaluation of mathematical reasoning with transformers.
\newblock {\em arXiv preprint arXiv:2305.12563}, 2023.

\bibitem[MX23]{mo2023tree}
Shentong Mo and Miao Xin.
\newblock Tree of uncertain thoughts reasoning for large language models, 2023.

\bibitem[NAGA{\etalchar{+}}21]{nye2021show}
Maxwell Nye, Anders~Johan Andreassen, Guy Gur-Ari, Henryk Michalewski, Jacob Austin, David Bieber, David Dohan, Aitor Lewkowycz, Maarten Bosma, David Luan, et~al.
\newblock Show your work: Scratchpads for intermediate computation with language models.
\newblock {\em arXiv preprint arXiv:2112.00114}, 2021.

\bibitem[NB18]{nangia2018listops}
Nikita Nangia and Samuel~R Bowman.
\newblock Listops: A diagnostic dataset for latent tree learning.
\newblock {\em arXiv preprint arXiv:1804.06028}, 2018.

\bibitem[NJL21]{nogueira2021investigating}
Rodrigo Nogueira, Zhiying Jiang, and Jimmy Lin.
\newblock Investigating the limitations of transformers with simple arithmetic tasks.
\newblock {\em arXiv preprint arXiv:2102.13019}, 2021.

\bibitem[PG23]{prystawski2023think}
Ben Prystawski and Noah~D Goodman.
\newblock Why think step-by-step? reasoning emerges from the locality of experience.
\newblock {\em 37th Conference on Neural Information Processing Systems (NeurIPS 2023)}, 2023.

\bibitem[PGZG23]{poesia2023certified}
Gabriel Poesia, Kanishk Gandhi, Eric Zelikman, and Noah~D Goodman.
\newblock Certified reasoning with language models.
\newblock {\em arXiv preprint arXiv:2306.04031}, 2023.

\bibitem[Pin99]{pinkus1999approximation}
Allan Pinkus.
\newblock Approximation theory of the mlp model in neural networks.
\newblock {\em Acta numerica}, 8:143--195, 1999.

\bibitem[PSL22]{press2021train}
Ofir Press, Noah~A Smith, and Mike Lewis.
\newblock Train short, test long: Attention with linear biases enables input length extrapolation.
\newblock {\em ICLR-2022}, 2022.

\bibitem[QWL{\etalchar{+}}22]{qian2022limitations}
Jing Qian, Hong Wang, Zekun Li, Shiyang Li, and Xifeng Yan.
\newblock Limitations of language models in arithmetic and symbolic induction.
\newblock {\em arXiv preprint arXiv:2208.05051}, 2022.

\bibitem[QXS{\etalchar{+}}23]{qi2023art}
Jingyuan Qi, Zhiyang Xu, Ying Shen, Minqian Liu, Di~Jin, Qifan Wang, and Lifu Huang.
\newblock The art of socratic questioning: Zero-shot multimodal reasoning with recursive thinking and self-questioning.
\newblock {\em arXiv preprint arXiv:2305.14999}, 2023.

\bibitem[RDG{\etalchar{+}}23]{ruoss2023randomized}
Anian Ruoss, Gr{\'e}goire Del{\'e}tang, Tim Genewein, Jordi Grau-Moya, R{\'o}bert Csord{\'a}s, Mehdi Bennani, Shane Legg, and Joel Veness.
\newblock Randomized positional encodings boost length generalization of transformers.
\newblock {\em arXiv preprint arXiv:2305.16843}, 2023.

\bibitem[SBS23]{stolfo2023understanding}
Alessandro Stolfo, Yonatan Belinkov, and Mrinmaya Sachan.
\newblock Understanding arithmetic reasoning in language models using causal mediation analysis.
\newblock {\em arXiv preprint arXiv:2305.15054}, 2023.

\bibitem[SDP{\etalchar{+}}23]{sun2022length}
Yutao Sun, Li~Dong, Barun Patra, Shuming Ma, Shaohan Huang, Alon Benhaim, Vishrav Chaudhary, Xia Song, and Furu Wei.
\newblock A length-extrapolatable transformer.
\newblock {\em Proceedings of the 61st Annual Meeting of the Association for Computational Linguistics}, 2023.

\bibitem[SH22]{saparov2022language}
Abulhair Saparov and He~He.
\newblock Language models are greedy reasoners: A systematic formal analysis of chain-of-thought.
\newblock {\em arXiv preprint arXiv:2210.01240}, 2022.

\bibitem[SHH22]{shao2022chaining}
Zhihong Shao, Fei Huang, and Minlie Huang.
\newblock Chaining simultaneous thoughts for numerical reasoning.
\newblock {\em arXiv preprint arXiv:2211.16482}, 2022.

\bibitem[SPBG23]{she2023scone}
Jingyuan~Selena She, Christopher Potts, Samuel~R Bowman, and Atticus Geiger.
\newblock Scone: Benchmarking negation reasoning in language models with fine-tuning and in-context learning.
\newblock {\em arXiv preprint arXiv:2305.19426}, 2023.

\bibitem[SPK{\etalchar{+}}23]{schaeffer2023invalid}
Rylan Schaeffer, Kateryna Pistunova, Samar Khanna, Sarthak Consul, and Sanmi Koyejo.
\newblock Invalid logic, equivalent gains: The bizarreness of reasoning in language model prompting.
\newblock {\em arXiv preprint arXiv:2307.10573}, 2023.

\bibitem[SSS{\etalchar{+}}22]{suzgun2022challenging}
Mirac Suzgun, Nathan Scales, Nathanael Sch{\"a}rli, Sebastian Gehrmann, Yi~Tay, Hyung~Won Chung, Aakanksha Chowdhery, Quoc~V Le, Ed~H Chi, Denny Zhou, et~al.
\newblock Challenging big-bench tasks and whether chain-of-thought can solve them.
\newblock {\em arXiv preprint arXiv:2210.09261}, 2022.

\bibitem[TDA{\etalchar{+}}21]{tay2021long}
Yi~Tay, Mostafa Dehghani, Samira Abnar, Yikang Shen, Dara Bahri, Philip Pham, Jinfeng Rao, Liu Yang, Sebastian Ruder, and Donald Metzler.
\newblock Long range arena: A benchmark for efficient transformers.
\newblock {\em Proceedings of ICLR2021}, 2021.

\bibitem[TNB23]{tan2023towards}
Qingyu Tan, Hwee~Tou Ng, and Lidong Bing.
\newblock Towards benchmarking and improving the temporal reasoning capability of large language models.
\newblock {\em arXiv preprint arXiv:2306.08952}, 2023.

\bibitem[TZL{\etalchar{+}}23]{tang2023large}
Xiaojuan Tang, Zilong Zheng, Jiaqi Li, Fanxu Meng, Song-Chun Zhu, Yitao Liang, and Muhan Zhang.
\newblock Large language models are in-context semantic reasoners rather than symbolic reasoners.
\newblock {\em arXiv preprint arXiv:2305.14825}, 2023.

\bibitem[WDZ{\etalchar{+}}23]{wang2023exploring}
Dingzirui Wang, Longxu Dou, Wenbin Zhang, Junyu Zeng, and Wanxiang Che.
\newblock Exploring equation as a better intermediate meaning representation for numerical reasoning.
\newblock {\em arXiv preprint arXiv:2308.10585}, 2023.

\bibitem[WL23]{wang2023learning}
Tianduo Wang and Wei Lu.
\newblock Learning multi-step reasoning by solving arithmetic tasks.
\newblock In {\em Proceedings of the 61st Annual Meeting of the Association for Computational Linguistics (Volume 2: Short Papers)}, pages 1229--1238, 2023.

\bibitem[WLC{\etalchar{+}}23]{wang2023making}
Peiyi Wang, Lei Li, Liang Chen, Feifan Song, Binghuai Lin, Yunbo Cao, Tianyu Liu, and Zhifang Sui.
\newblock Making large language models better reasoners with alignment.
\newblock {\em arXiv preprint arXiv:2309.02144}, 2023.

\bibitem[WLS23]{wies2023sub}
Noam Wies, Yoav Levine, and Amnon Shashua.
\newblock Sub-task decomposition enables learning in sequence to sequence tasks.
\newblock {\em Proceddings of International Conference on Learning Representations (ICLR-2023)}, 2023.

\bibitem[WMD{\etalchar{+}}22]{wang2022towards}
Boshi Wang, Sewon Min, Xiang Deng, Jiaming Shen, You Wu, Luke Zettlemoyer, and Huan Sun.
\newblock Towards understanding chain-of-thought prompting: An empirical study of what matters.
\newblock {\em arXiv preprint arXiv:2212.10001}, 2022.

\bibitem[WQR{\etalchar{+}}23]{wu2023reasoning}
Zhaofeng Wu, Linlu Qiu, Alexis Ross, Ekin Aky{\"u}rek, Boyuan Chen, Bailin Wang, Najoung Kim, Jacob Andreas, and Yoon Kim.
\newblock Reasoning or reciting? exploring the capabilities and limitations of language models through counterfactual tasks.
\newblock {\em arXiv preprint arXiv:2307.02477}, 2023.

\bibitem[WSC{\etalchar{+}}23]{wang2023boosting}
Jianing Wang, Qiushi Sun, Nuo Chen, Xiang Li, and Ming Gao.
\newblock Boosting language models reasoning with chain-of-knowledge prompting.
\newblock {\em arXiv preprint arXiv:2306.06427}, 2023.

\bibitem[WWS{\etalchar{+}}22a]{wang2022self}
Xuezhi Wang, Jason Wei, Dale Schuurmans, Quoc Le, Ed~Chi, Sharan Narang, Aakanksha Chowdhery, and Denny Zhou.
\newblock Self-consistency improves chain of thought reasoning in language models.
\newblock {\em arXiv preprint arXiv:2203.11171}, 2022.

\bibitem[WWS{\etalchar{+}}22b]{wei2022chain}
Jason Wei, Xuezhi Wang, Dale Schuurmans, Maarten Bosma, Fei Xia, Ed~Chi, Quoc~V Le, Denny Zhou, et~al.
\newblock Chain-of-thought prompting elicits reasoning in large language models.
\newblock {\em Advances in Neural Information Processing Systems}, 35:24824--24837, 2022.

\bibitem[XLH{\etalchar{+}}23]{xu2023large}
Fangzhi Xu, Qika Lin, Jiawei Han, Tianzhe Zhao, Jun Liu, and Erik Cambria.
\newblock Are large language models really good logical reasoners? a comprehensive evaluation from deductive, inductive and abductive views.
\newblock {\em arXiv preprint arXiv:2306.09841}, 2023.

\bibitem[XPL{\etalchar{+}}23]{xu2023rewoo}
Binfeng Xu, Zhiyuan Peng, Bowen Lei, Subhabrata Mukherjee, Yuchen Liu, and Dongkuan Xu.
\newblock Rewoo: Decoupling reasoning from observations for efficient augmented language models.
\newblock {\em arXiv preprint arXiv:2305.18323}, 2023.

\bibitem[YLZ23]{yao2023beyond}
Yao Yao, Zuchao Li, and Hai Zhao.
\newblock Beyond chain-of-thought, effective graph-of-thought reasoning in large language models.
\newblock {\em arXiv preprint arXiv:2305.16582}, 2023.

\bibitem[YSAN22]{yang2022chain}
Mengjiao~Sherry Yang, Dale Schuurmans, Pieter Abbeel, and Ofir Nachum.
\newblock Chain of thought imitation with procedure cloning.
\newblock {\em Advances in Neural Information Processing Systems}, 35:36366--36381, 2022.

\bibitem[YTL{\etalchar{+}}23]{yao2023thinking}
Fanglong Yao, Changyuan Tian, Jintao Liu, Zequn Zhang, Qing Liu, Li~Jin, Shuchao Li, Xiaoyu Li, and Xian Sun.
\newblock Thinking like an expert: Multimodal hypergraph-of-thought (hot) reasoning to boost foundation modals.
\newblock {\em arXiv preprint arXiv:2308.06207}, 2023.

\bibitem[YYZ{\etalchar{+}}23]{yao2023tree}
Shunyu Yao, Dian Yu, Jeffrey Zhao, Izhak Shafran, Thomas~L Griffiths, Yuan Cao, and Karthik Narasimhan.
\newblock Tree of thoughts: Deliberate problem solving with large language models.
\newblock {\em arXiv preprint arXiv:2305.10601}, 2023.

\bibitem[ZBB{\etalchar{+}}22]{zhang2022unveiling}
Yi~Zhang, Arturs Backurs, S{\'e}bastien Bubeck, Ronen Eldan, Suriya Gunasekar, and Tal Wagner.
\newblock Unveiling transformers with lego: a synthetic reasoning task.
\newblock {\em arXiv preprint arXiv:2206.04301}, 2022.

\bibitem[ZBL{\etalchar{+}}23]{zhou2023algorithms}
Hattie Zhou, Arwen Bradley, Etai Littwin, Noam Razin, Omid Saremi, Josh Susskind, Samy Bengio, and Preetum Nakkiran.
\newblock What algorithms can transformers learn? a study in length generalization.
\newblock {\em arXiv preprint arXiv:2310.16028}, 2023.

\bibitem[ZYYY23]{zhang2023cumulative}
Yifan Zhang, Jingqin Yang, Yang Yuan, and Andrew Chi-Chih Yao.
\newblock Cumulative reasoning with large language models.
\newblock {\em arXiv preprint arXiv:2308.04371}, 2023.

\end{thebibliography}


\newpage
\appendix

\vspace{4mm}
\subsection{Proofs - Causal Function}
\label{sec.dynamic.causal}
\vspace{+2mm}

We firstly provide Lemma \ref{lem: interpolation}, which is fundamental to our subsequent proofs. Lemma \ref{lem: interpolation} is similar to the \textit{universal approximation theorem} (UAT)~\cite{haykin1998neural,hassoun1995fundamentals,pinkus1999approximation}. We do not directly use the UAT as it is for neural networks but our results are not bound to any specific learning paradigm or algorithm. However, due to the UAT, our results apply to learning in neural networks. The UAT guarantees that $\sup_{x \in K}|| f(x) - g(x) || < \eta$, where $g(x)$ is the target function and $K$ is compact. Applying the UAT in our setting of Lemma \ref{lem: interpolation}, we have $\sup_{1 \leq i \leq n}|| f(x_i) - y_i || < \eta$, where $n$ is the number of training samples. The UAT has a condition that $K$ is compact. In Lemma \ref{lem: interpolation}, since we have a stronger condition that $K = \{(x_i, y_i) |\, 1 \leq i \leq n\}$ is finite, we provide a stronger result that $\sup_{1 \leq i \leq n}|| f(x_i) - y_i || = 0$.

\begin{lemma}[A Simple Interpolating Function]
Let $(X, d_X)$ be a metric space. Let $(Y, ||\cdot||)$ be a Banach space. 
For any $D \subseteq \{(x, y) |\, x \in X, y\in Y\}$, if $x \neq x',\, \forall\, (x, y), (x',y') \in D$, then there exists a continuous approximation function $f: X \rightarrow Y$, s.t. $f(x) = y,\,\forall\, (x, y) \in D$.
\label{lem: interpolation}
\end{lemma}

\begin{proof}[Proof of Lemma \ref{lem: interpolation}]
Let 
$$
\epsilon = \min_{(x, y), (x', y') \in D} d_X(x, x').
$$
Define 
$$
K_\epsilon (x, x') = \frac{\epsilon}{d_X (x, x')}.
$$
Denote $D = \{(x_1, y_1), \dots, (x_n, y_n)\}$, let
\begin{equation}
    f(x) = \frac{\sum_{i = 1}^n y_i K_\epsilon(x, x_i)}{\sum_{i = 1}^n K_\epsilon(x, x_i)},\, x \in X \backslash \{x_1, \dots, x_n\}.
\label{eq: interpolate}
\end{equation}
For $\forall\, 1 \leq i_0 \leq n$  ($i_0$ is any integer), since $\lim_{x \rightarrow x'} K_\epsilon (x, x') = +\infty$, $\sup_{d_X(x, x')  > \epsilon} K_\epsilon (x, x') < 1$ and 
$$
f(x) = \frac{ y_{i_0} K_\epsilon(x, x_{i_0})}{\sum_{i = 1}^n K_\epsilon(x, x_i)} + \frac{\sum_{i \neq i_0} y_i K_\epsilon(x, x_i)}{\sum_{i = 1}^n K_\epsilon(x, x_i)},
$$  
it's obvious that $\forall\, \eta > 0$, $\exists\, \delta > 0$, s.t. $\forall\, d_X(x, x_{i_0}) < \delta$, $||f(x) - y_{i_0}|| < \eta$, which is $\lim_{x \rightarrow x_{i_0}} f(x) = y_{i_0}$. 
Therefore, we can define $f$ on $X$ and it's obvious that $f$ is a continuous approximation function. 
\end{proof}

~\\

\begin{proof}[\bf Proof of Theorem \ref{thm: full_causal_function}]
By Lemma \ref{lem: interpolation}, let $X^{\sup |p(v)|}$ be $X$ and $X$ be $Y$, there exists $\Hat{f}$ s.t. $\Hat{f}(p(v)) = v',\,\forall\, (v', p(v)) \in D$. 
Since $D = \{(f(p(v)), p(v)) |\, v \in V\}$, the proof is done. 
\end{proof}

~\\

\begin{proof}[\bf Proof of Corollary \ref{thm: not_full_causal_function}]
In the corollary presented in Section~\ref{sec: dag} in the main text, we wrote $D \neq X^{\sup |p(v)|}$ for simplicity. 
In detail, we have $D \neq \{(f(p(v)), p(v)) |\, v \in V\}$. 
For simplicity, denote $A = \{(f(p(v)), p(v)) |\, v \in V\}$. 

Since $D \neq A$, let $(f(s(v_0)), s(v_0)) \in A \backslash D$. 
Since $|\{f(p(v)|\, p(v) \in X^{\sup |p(v)|})\}| > 1$, let $v^{err} \in \{f(p(v)|\, p(v) \in X^{\sup |p(v)|})\}$ s.t. $v^{err} \neq f(s(v_0))$. 
Let $$D^{err} = D \cup \{(v^{err}, s(v_0))\}.$$

By Lemma \ref{lem: interpolation}, let $X^{\sup |p(v)|}$ be $X$ and $X$ be $Y$, there exists $\Hat{f}$ s.t. $\Hat{f}(p(v)) = v',\,\forall\, (v', p(v)) \in D^{err}$. 
Since $(v^{err}, s(v_0)) \in D^{err}$, $\Hat{f}$ makes a wrong prediction as $$\Hat{f}(s(v_0)) = v^{err} \neq f(s(v_0)). $$
Since $D \subseteq D^{err}$, $\Hat{f}$ is correct on training dataset as $$\Hat{f}(p(v)) = f(p(v)),\,\forall\, (f(p(v)), p(v)) \in D.$$ 
\end{proof}
 
 ~\\

\begin{proof}[\bf Proof of Corollary \ref{thm: not_full_causal_function_inf_x}]
Denote $D \subseteq \{(f(p(v)), p(v)) |\, v \in V\}$. 
For simplicity, denote $A = \{(f(p(v)), p(v)) |\, v \in V\}$. 

Since $\max(|X|, \sup |p(v)|) = \infty$, $|X^{\sup |p(v)|}| = \infty$, $|A| = \infty$. Since $|D| < \infty$, we know $D \neq A$ and $|A \backslash D| = \infty$. 
Let $(f(s(v_i), v_i)) \in A \backslash D,\,i=0,1\dots,m$. Since $|\{f(p(v)|\, p(v) \in X^{\sup |p(v)|})\}| > 1$, let $v_i^{err} \in \{f(p(v)|\, p(v) \in X^{\sup |p(v)|})\}$ s.t. $v^{err}_i \neq f(s(v_i))$. 
Let $$D^{err} = D \cup \{(f(s(v_i), v_i)),\, i = 0, \dots, m)\}.$$

By Lemma \ref{lem: interpolation}, let $X^{\sup |p(v)|}$ be $X$ and $X$ be $Y$, there exists $\Hat{f}$ s.t. $\Hat{f}(p(v)) = v',\,\forall\, (v', p(v)) \in D^{err}$. 
Since $(v_i^{err}, s(v_i)) \in D^{err}$, $\Hat{f}$ makes $m + 1$ mistakes as $$\Hat{f}(s(v_i)) = v_i^{err} \neq f(s(v_i)), \,i = 0, 1, \dots, m. $$
Since $D \subseteq D^{err}$, $\Hat{f}$ is correct on training dataset as $$\Hat{f}(p(v)) = f(p(v)),\,\forall\, (f(p(v)), p(v)) \in D.$$ 
\end{proof}

\vspace{+2mm}
\subsection{Proofs - Recursive Formula}
\label{sec.recursive.direct}
\vspace{+2mm}

\begin{proof}[\bf Proof of Theorem \ref{thm: recursive_dag_solved}]
By Theorem \ref{thm: full_causal_function}, there exists $\Hat{f}$, s.t. $\Hat{f}(p(v)) = f(p(v)),\,\forall\, (f(p(v)), p(v)) \in \{(f(p(v)), p(v)) |\, v \in V\}$. 

By the definition in Eq. \eqref{eq: def_dag} in Section~\ref{sec: dag}, for an arbitrary topological sorting of the graph $G$, we have 
$$
G_f(\{v_i, i \leq |G| |\, d(v_i) = 0\}) = (v_1, \dots, v_{|G|}), \text{where}\ 
\left\{
\begin{aligned}
&v_1 = f(s(v_1)), \\
&\dots \\
&v_n = f(s(v_n)).
\end{aligned}
\right.
$$

Given the graph $G$ and the topological sorting, by induction, we have
$$
\begin{aligned}
v_1 = f(s(v_1)) &= \Hat{f}(s(v_1)) = \Hat{v}_1, \\
&\dots \\
v_n = f(s(v_n)) &= \Hat{f}(s(\Hat{v}_n)) = \Hat{v}_n. 
\end{aligned}
$$
\end{proof}

\vspace{2mm}
\subsection{Proof - Maximal Input Element Distance of a Reasoning Step}
\label{sec.mied}

\begin{figure}

\caption{
Two examples of the \textit{ko} problem. 
} 
\center
\includegraphics[width=0.5\linewidth]{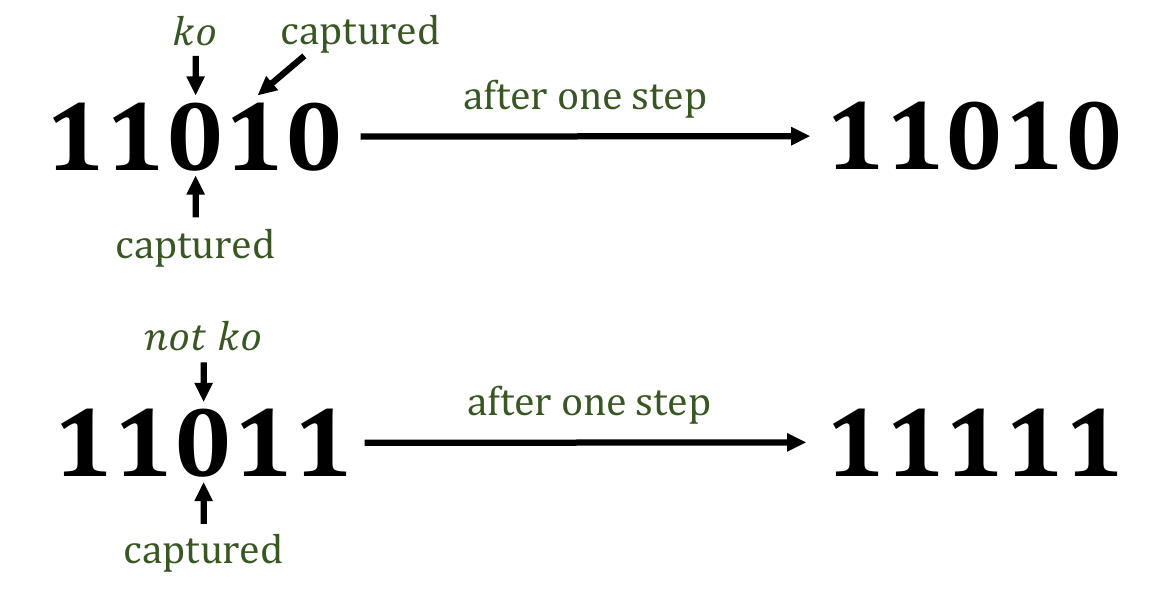}
\label{fig:example_ko}
\vspace{-2mm}
\end{figure}

Before presenting the proof for Theorem~\ref{thm: find_input_vertices_is_possible}, we first provide an example to show that a $4R$-length window is not enough to decide which elements should be used in the next reasoning or calculation step. 

Let's consider a simple one-dimensional \textit{ko} problem. Let $s = \dots 0 0 1 0 0 1 1 1 1 \dots$ be a sequence of $0$'s and $1$'s. We say $s_i$ is captured if $s_{i} \neq s_{i-1},\,s_{i} \neq s_{i+1}$. We say $s_i$ is not a ko if both $s_{i-1}$ and $s_{i+1}$ are not captured. We say $s_i$ is a ko if $s_i$ is captured and at least one of $s_{i-1}$ and $s_{i+1}$ is captured. When some $s_i$ is captured and is not a ko, {we set $s_i = 1 - s_i$.} We continue the process until $s$ won't change, which is the final settlement of the problem.
In this problem, if $s_i$ is captured and is not a ko, {it should be acted upon next (changing $s_i$'s value).} The neighbors $s_{i-1}$ and $s_{i+1}$ are enough to decide the value of $s_i$, i.e. $\max_{v' \in (p \circ t)(s_i)} d(v', s_i) = \max(d(s_{i-1}, s_i), d(s_{i+1}, s_i)) = 1$. Therefore, $R = 1$ for this problem. 

However, a $4R$-length window is not enough to decide whether $s_i$ should be reasoned or acted upon next. For instance, let $s = 1 1 0 1 0$. In this problem, $s_{[2]} = 0$ is captured and is a ko, so it cannot be changed. Therefore, when we only consider $4$ elements $s_{[0:4]} = 1 1 0 1$, we will receive the ground truth label $y_{[0:4]}= 0 0 0 0$, where $0$ indicates the element shouldn't be changed and $1$ indicates the element should be changed. However, when we consider a problem $s' = 1 1 0 1 1$, $s_{[2]}' = 0$ is captured and is not a ko, so it should be changed. Therefore, when we only consider $4$ elements $s_{[0:4]}' = 1 1 0 1$, the ground truth label is $y_{[0:4]}'= 0 0 1 0$. These two problems share an identical $4R$-length window $1 1 0 1$, but the labels of this window are different, which is obviously ambiguous. 

\vspace{+2mm}
\begin{proof}[\bf Proof of Theorem \ref{thm: find_input_vertices_is_possible}]
Denote the training dataset as
$$
D \subseteq \{(\Tilde{g}(s), s) | s = s_0 s_1 \dots s_{4R}, s = g^{-1}(G)\}, 
$$
where $\Tilde{g}(s) = 1_{2R}(s_{2R})$ is a binary label that $1_{2R} (s_{2R})$ indicates whether the central element of the $4R+1$ interval should be reasoned next. 

By Lemma \ref{lem: interpolation}, let $X^{4R + 1}$ to be $X$ and $\{0, 1\}$ to be $Y$, there exists $\Hat{f}$ s.t. $\Hat{f}(s) = I,\,\forall\, (I, p(v)) \in D$. 
Since $D = \{(\Tilde{g}(s), s) | s = s_0 s_1 \dots s_{4R}, s = g^{-1}(G)\}$, the proof is done. 
\end{proof}

\vspace{+2mm}

\begin{proof}[\bf Proof of Theorem \ref{thm: (1,4R+1)-consist}]

\textbf{(1)} Because $R < \infty$, it's obvious that an interval of length $4R + 1$ can cover the input elements of one reasoning step. 

\textbf{(2)} For any $(4R + 1)$-length interval $I = s_{j_1} \dots s_{j_{4R + 1}}$, denote the central element $s_{j_{2R}}$ as $s_c$. 
For any $s^0, s^1 \in \{s| s = s_0 \dots s_n \supseteq I\}$, i.e. $I$ is a sub-interval of $s^0$ and $s^1$, we prove the problem is $(1, 4R+1)$-consistent by showing $$\Tilde{g}(s^0)|_{s_c} = \Tilde{g}(s^1)|_{s_c}. $$

Since we have two $DAG$'s $g(s^0)$ and $g(s^1)$ in the proof, we will always write $\subseteq g(s^0)$ or $\subseteq g(s^1)$ after a set to distinguish the graph. 

i) When $s_c \in \Tilde{g}(s^0)$, assume $s_c \notin \Tilde{g}(s^1)$. Since $s_c \in \Tilde{g}(s^0)$, by definition \eqref{eq: input_elements}, there exists $v_0 \in t(s_c) \subseteq g(s^0)$ s.t. $\forall\, w \in p(v),\, |s(w)| = 0$. By definition of $R$, we know $p(v_0) \subseteq s_{j_R} \dots s_{j_{3R}} \subseteq I$.
Since $s_c \notin \Tilde{g}(s^1)$, there exists $w_0 \in p(v_0) \subseteq g(s^0)$ s.t. $\forall\, v \in t(s_c) \subseteq g(s^1),\, w_0 \notin p(v) \subseteq g(s^1)$, i.e. $v_0 \notin s(w_0) \subseteq g(s^1)$. 
Since $w_0 \in p(v_0) \subseteq s_{j_R} \dots s_{j_{3R}} \subseteq I$, for $\forall\, v_1 \in t(w_0)$, by definition of $R$, we know $p(v_1) \subseteq s_{j_0} \dots s_{j_{4R}} \subseteq I$. Therefore, we find $p(v_0), p(v_1) \subseteq I,\,v_0 \neq v_1$ and $w_0 \in p(v_0) \subseteq g(s^0)$ and $w_0 \in p(v_1) \subseteq g(s^1)$, which means that $w_0$ belongs to two different sets of elements in $I$ that can be reasoned. Contradiction! Therefore, $s_c \in \Tilde{g}(s^1)$. 

ii) When $s_c \in \Tilde{g}(s^1)$, similarly, we have $s_c \in \Tilde{g}(s^0)$. Taking the contrapositive, when $s_c \notin \Tilde{g}(s^0)$, $s_c \notin \Tilde{g}(s^1)$. 

Therefore, $\Tilde{g}(s^0)|_{s_c} = \Tilde{g}(s^1)|_{s_c}$ always holds. 
\end{proof}


\vspace{2mm}
\subsection{Proof - $(n, r)$-Consistency}
\label{sec.(n,r)-consistency}

\begin{proof}[\bf Proof of Theorem \ref{thm:  n,r_implies_n,r'}]
Since the problem is $(n, r)$-consistent, we verify that the problem is also $(n, r')$-consistent, $\forall\, r' \geq r$. 

\textbf{(1)} It's obvious that when the input elements can be covered by a set of $n$ $r$-length intervals, they can also be covered by a set of $n$ $r'$-length intervals, $\forall\, r' \geq r$. 

\textbf{(2)} For $\forall\, \{I_1, \dots, I_n | I_j = s_{j_1} \dots s_{j_{r'}}, j = 1, \dots, n\}$, $\forall s^0, s^1$ that contains these $n$ $r'$-length contiguous sub-sequences i.e. $\forall\, s^0, s^1 \in \{s | s = s_1 s_2 \dots s_{i_0} \supseteq I_j,\, j = 1, \dots, n\}$, we want to show 
$$
\Tilde{g}(s^0)|_{\{s_{j_{c}} |\, j = 1, \dots, n\}} = \Tilde{g}(s^1)|_{\{s_{j_{c}} |\, j = 1, \dots, n\}},
$$
where $s_{j_{c}} = s_{j_{\lfloor\frac{r' + 1}{2}\rfloor}}$. 

Since the problem is $(n, r)$-consistent, we consider $n$ $r$-length intervals 
$$\begin{aligned}
&\{K_1, \dots, K_n | K_j = s_{j_{1 + [\frac{r' - r}{2}]}} \dots s_{j_{r' - [\frac{r' - r}{2}]}}, j = 1, \dots, n\}, &r~mod~2 = r'~mod~2, \\
&\{K_1, \dots, K_n | K_j = s_{j_{1 + [\frac{r' - r + 1}{2}]}} \dots s_{j_{r' - [\frac{r' - r}{2}]}}, j = 1, \dots, n\}, &r~mod~2 = 0,\,r'~mod~2 = 1, \\
&\{K_1, \dots, K_n | K_j = s_{j_{1 + [\frac{r' - r}{2}]}} \dots s_{j_{r' - [\frac{r' - r + 1}{2}]}}, j = 1, \dots, n\}, &r~mod~2 = 1,\,r'~mod~2 = 0. \\
\end{aligned}$$
It's not difficult to verify that the central element of $K_j$ is exactly the central element of $I_j$ i.e. $s_{j_c}$. 
Since the problem is $(n, r)$-consistent, applying the definition on $\{K_1, \dots, K_n \}$ and $s^0, s^1$, we have $$
\Tilde{g}(s^0)|_{\{s_{j_{c}} |\, j = 1, \dots, n\}} = \Tilde{g}(s^1)|_{\{s_{j_{c}} |\, j = 1, \dots, n\}}. 
$$
The proof is done. 
\end{proof}

\begin{proof}[\bf Proof of Theorem \ref{thm: recover_g_by_gamma}]
For $\forall\, s^0 = s^0_0 s^0_1 \dots s^0_n$, for ease of notation, when we write $s^0_i$ for $i < 0$ or $i > n$, we mean $s^0_i = \text{` '}$, which is an empty padding token. By definition of $\gamma$, we have $\gamma(I_1, \dots, I_n) = \Tilde{g}(s^0)|_{\{s_{j_{c}} |\, j = 1, \dots, n\}}$ for any $I_1, \dots, I_n$ to be sub-intervals of $s^0$. 

For $\forall\, s^0_{i_0} \in s^0$, define 
$$\Hat{g}(s^0_{i_0}) = \gamma(I_1, \dots, I_n)|_{s^0_{i_0}}, $$
where $I_1 = s^0_{i_0 - [\frac{r - 1}{2}]} \dots s^0_{i_0 + [\frac{r}{2}]}$, and $I_2, \dots, I_n$ are arbitrary sub-intervals of $s^0$. It's easy to verify that $s^0_{i_0}$ is the central element of $I_1$, i.e. $s^0_{i_0} = s_{1_c}$. 
Since $\gamma(I_1, \dots, I_n) = \Tilde{g}(s^0)|_{\{s_{j_{c}} |\, j = 1, \dots, n\}}$, we have 
$$\Hat{g}(s^0_{i_0}) = \gamma(I_1, \dots, I_n)|_{s^0_{i_0}} = \gamma(I_1, \dots, I_n)|_{s_{1_c}} = \Tilde{g}(s^0)|_{s_{1_c}} = \Tilde{g}(s^0)|_{s^0_{i_0}}. $$
Traversing $s^0$ by letting $s^0_{i_0} = s^0_{1}, \dots, s^0_{n}$, we have found all the elements involved in the next reasoning step, i.e. $$\bigcup \Tilde{g}(s) = \bigcup_{i=1}^{n} \Tilde{g}(s^0)|_{s^0_{i}} = \bigcup_{i=1}^{n} \Hat{g}(s^0_{i}). $$
\end{proof}

\begin{proof}[\bf Proof of Theorem \ref{thm: learn_gamma_is_possible}]
By Lemma \ref{lem: interpolation}, let $X^{r \times n}$ be $X$ and $2^n$ be $Y$, $D = \{(\gamma(I_1, \dots, I_n), \{I_1, \dots, I_n\})\}$,  there exists $\Hat{\gamma}$ s.t. $\Hat{\gamma}(I_1, \dots, I_n) = \gamma(I_1, \dots, I_n),\,\forall\, (\gamma(I_1, \dots, I_n), \{I_1, \dots, I_n\}) \in D$. 
\end{proof}


\begin{table}[h]
\fontsize{9pt}{9pt}\selectfont 
\caption{
Experimental examples of the CoT process. See Appendix~\ref{appendix.examples} for explanations of the symbols used.
} 
\vspace{3mm}
\centering
\begin{minipage}[l]{0.35\textwidth}
\subfigure[arctan]{
\scalebox{1.0}{
\begin{tabular}{l|l}
\toprule
Input[0]:     & a, b \\
Output[0]:      & arctan(a/b) \\
\bottomrule
\end{tabular}
}
}
\subfigure[arithmetic in prime field $F_7$]{
\scalebox{1.0}{
\begin{tabular}{l|l}
\toprule
Input[0]:      &  (0+4-(2-3*6))*(4+0) \\
Output[0]:      &  (\ \,4~\,-(2-~4~~))*\ \ \,4 \\
\midrule
Input[1]:      & (4-(2-4))*4  \\
Output[1]:      & (4-\ \ \,5 \ )*4 \\
\midrule
Input[2]:      & (4-5)*4  \\
Output[2]:      & \ ~6\ ~*4 \\
\midrule
Input[3]:      &  6*4 \\
Output[3]:      & \ \,3 \\
\bottomrule
\end{tabular}
}
}
\subfigure[multiplication (1-line)]{
\scalebox{1.0}{
\begin{tabular}{l|l}
\toprule
Input[0]:  &  1*3=? \\
Output[0]: & 1*3=1+? \\ 
\midrule
Input[1]:  &  1*3=1+? \\
Output[1]: & 1*3=1+1+? \\ 
\midrule
Input[2]:  &  1*3=1+1+? \\
Output[2]: & 1*3=1+1+1 \\ 
\midrule
Input[3]:  &  1*3=1+1+1 \\
Output[3]: & 1*3=2+1 \\
\midrule
Input[4]:  &  1*3=2+1 \\
Output[4]: & 1*3=3 \\
\bottomrule
\end{tabular}
}
}
\end{minipage}
\begin{minipage}[l]{0.35\textwidth}
\scalebox{1.0}{
\subfigure[addition (1-line)]{
\begin{tabular}{l|l}
\toprule
Input[0]:  &  285+9805=? \\
Output[0]: &  285+9805=c0 \\
\midrule
Input[1]:  &  285+9805=c0 \\ 
Output[1]: & 285+9805=?90 \\
\midrule
Input[2]:  &  285+9805=?90 \\
Output[2]: & 285+9805=c090 \\
\midrule
Input[3]:  &  285+9805=c090 \\ 
Output[3]: & 285+9805=10090 \\
\bottomrule
\end{tabular}
}
}
\scalebox{1.0}{
\subfigure[addition (2-line)]{
\begin{tabular}{l|l}
\toprule
Input[0]:  &  ~~285+~~9805=~~~~~~? \\
           &  ~~~~~I~~~~~~\,~~i~~~~~~\,~~J \\
Output[0]: &  ~~285+~~9805=~~~~~c0 \\ 
           &  ~~~~I~~~~~~\,~~i~~~~~~\,~~J \\
\midrule
Input[1]:  &  ~~285+~~9805=~~~~~c0 \\ 
           &  ~~~~I~~~~~~\,~~i~~~~~~\,~~J \\
Output[1]: &   ~~285+~~9805=~~~~?90 \\
           &  ~~I~~~~~~\,~~i~~~~~~~\,~~J \\
\midrule
Input[2]:  &  ~~285+~~9805=~~~~?90 \\
           &  ~~I~~~~~~\,~~i~~~~~~~\,~~J \\
Output[2]: &  ~~285+~~9805=~~c090 \\
           &  I~~~~~~~~~i~~~~~~~~J \\
\midrule
Input[3]:  &  ~~285+~~9805=~~c090 \\ 
           &  I~~~~~~~~~i~~~~~~~~J \\
Output[3]: &  ~~285+~~9805=10090 \\
           &  I~~~~~~~~i~~~~\,~~~J\,~ \\
\bottomrule
\end{tabular}
}
}
\end{minipage}
\begin{minipage}[r]{0.28\textwidth}
\subfigure[addition (3-line)]{
\scalebox{0.95}{
\begin{tabular}{l|r}
\toprule
Input[0]:  &   89283 \\
  &     3360 \\
  &          ? \\
Output[0]:   &     ?3 \\
\midrule
Input[1]:   &  89283 \\
   &    3360 \\
   &        ?3 \\
Output[1]:	&  c43 \\
\midrule
Input[2]:   &  89283 \\ 
   &    3360 \\
   &      c43 \\
Output[2]:  &  ?643 \\
\midrule
Input[3]:   &  89283 \\
   &    3360 \\ 
  &    ?643 \\
Output[3]:  & c2643 \\
\midrule
Input[4]:   &  89283 \\
   &    3360 \\
   &  c2643 \\
Output[4]:  & 92643 \\
\bottomrule
\end{tabular}
}
}
\scalebox{0.95}{
\subfigure[parity (2-line)]{
\begin{tabular}{l|l}
\toprule
Input[0]:  &  1011 \\
           &  ? \\
Output[0]: &  1? \\ 
\midrule
Input[1]:  &  1011 \\ 
           &  1? \\
Output[1]: &  11? \\
\midrule
Input[2]:  &  1011 \\ 
           &  11? \\
Output[2]: &  110? \\
\midrule
Input[3]:  &  1011 \\ 
           &  110? \\
Output[3]: &  1101 \\
\bottomrule
\end{tabular}
}
}
\end{minipage} 
\\
\begin{minipage}[c]{1.0\textwidth}
\vspace{4mm}
\subfigure[multiplication (8-line)]{
\scalebox{0.9}{
\begin{tabular}{l|p{0.1mm}p{0.1mm}p{0.1mm}p{1mm}|l|p{0.1mm}p{0.1mm}p{0.1mm}p{0.1mm}p{1mm}|l|p{0.1mm}p{0.1mm}p{0.1mm}p{0.1mm}p{0.1mm}p{1mm}}
\toprule
Input[0]:  &  & 2 & 3 & 4 & Input[2]: &  &  & 2 & 3 & 4 & Input[4]: &  & &  & 2 & 3 & 4 \\
  &  E &  & & S & & & E & & & S    & & & & E & & & S\\
  &   &   & &  I & & &  &  & I &   & & & &  & I & & \\
  & & & 5 & 6    & & &  & & 5 & 6  & & & &  & & 5 & 6\\
  & & e & & s    & & &  & e & & s  & & & &  & e & & s\\
  & & & & i      & & & & & & i     & & & & & & & i \\
  & & F & & T    & & & F & & T &   & & & F & & T & & \\
  & & & & J      & & & & & J &     & & & & & J & &\\
Output[0]:  & & & & I & Output[2]: & & & & I & & Output[4]: & & & & I & &\\
  & & & i & & & & & & i &     & & & & & & i &\\
  & & F & & T & & & F & & T   & & & & F & & T & &\\
  & & & J & & & &  & J & &    & & & & J & & &\\
  & & & 2 & 4 & & & & 1 & 8 & & & & & 1 & 2 & & \\
Answer[0]: & & & 2 & 4 & Answer[2]: & & & 4 & 0 & 4 & Answer[4]: & & & 3 & 1 & 0 & 4 \\
\midrule
Input[1]:  & & 2 & 3 & 4 & Input[3]: & &  & 2 & 3 & 4  & Input[5]: & & &  & 2 & 3 & 4\\
  &   E &  & & S & & & E & & & S    & & & & E & & & S\\
  &   &   &  & I & & & &   & I &    & & & & & I & &\\
  & & & 5 & 6 & & & & & 5 & 6       & & & & & & 5 & 6\\
  & &  e & & s & & & &  e & & s     & & & & &  e & & s\\
  & & & i &  &  & & & & i &         & & & & & & i & \\
  & & F & & T & & & F & & T &       & & & F & & T & &\\
  & & & J &  & & & & J & &          & & & & J & & & \\
Output[1]: & & & I & & Output[3]: & & & I & & & Output[5]: & & & I & & &\\
  & & &  & i & & & & & & i       & & & & & & & i\\
  & F & & T & & & F & & T & &    & & F & & T & & & \\
  & & & J &  & & & & J & &       & & & & J & & & \\
  & & 2 & 0 & & & & 1 & 5 & &    & & & 1 & 0 & & & \\
Answer[1]: & & 2 & 2 & 4 & Answer[3]: & & 1 & 9 & 0 & 4 & Answer[5]: & & 1 & 3 & 1 & 0 & 4 \\
\bottomrule
\end{tabular}
}
}
\end{minipage}
\label{tab:examples}
\end{table}

\vspace{2mm}
\subsection{Experimental Data and Implementation Details}
\label{appendix.examples}

Table~\ref{tab:examples} shows some examples of CoT formulations of the experimental reasoning problems. The \textbf{arithmetic} problem is defined in the prime field $F_7$, where the calculations are under the sense of `\textbf{mod 7}'. The \textbf{parity} problem uses `?' in the second line of the input and output to represent the position to be calculated next, $1$ to represent \textit{odd} and $0$ to represent \textit{even}. The first line in the input is the input bit sequence. 
The \textbf{addition} problems use `?' to represent $0$ being carried from the right and `c' to represent $1$ being carried from the right. In all problems, * is equivalent to $\times$. We use * instead of $\times$ for ease of aligning chars. 

The \textbf{multiplication} problems use `E', `S', `e', `s', `I', `i', `F', `T', and `J' to represent specific positions of the CoT process, where some descriptions have been described in Sec. \ref{sec:experiment}. For each CoT step, `i' moves one position left. When `i' is at the same position as `e', it moves to the position of `s'. When `i' moves back to `e', `I' moves one position to the left. When `I' is at the position of `E', the CoT process ends. `F' and `T' are at the same positions as `e' and `s' at the beginning of the CoT process. For each CoT step, `J' moves one position to the left. When `J' is at the same position as `F', it moves to the position of `T'. Before `J' moves back to `J', `F' and `T' move one position to the left. Each CoT step predicts the next position of `I', `i', `F', `T', `J' and the result. The next CoT step uses the predicted `I', `i', `F', `T', `J' as the input. 

\textbf{Dataset generation:} Both the training dataset and each test dataset are composed of the following length proportions. For the \textbf{arctan} problem, each problem instance is uniformly sampled from the annulus region. For the \textbf{arithmetic} problem, the length of a problem instance in the dataset is generated to be as close to the maximum length as possible. For the \textbf{parity} problem, the length of a problem instance in the dataset is uniformly sampled from $1$ to the maximum length. For the \textbf{addition} problem, the length of the adders is uniformly sampled from $1$ to the maximum length.  For the \textbf{multiplication} problem, the length of the multipliers is also uniformly sampled from $1$ to the maximum length. 

\textbf{Training:} All the problems are trained with $50k$ batches. Each batch contains $256$ CoT steps. For each problem, we first randomly generate an instance of the problem and its detailed CoT steps. Each CoT step is a pair (Input[i], Output[i]), as shown in Table \ref{tab:examples}. We put the CoT steps into a batch until the batch size reaches $256$. 

\textbf{Testing:} When testing the performance after the training, we test $6$ different datasets for each problem, which is shown in Table \ref{tab:experiment_setting} in the main text. Each testing dataset is independently generated with \textbf{$1k$ questions}. We solve each question by CoT using the trained model.
For \textit{arctan}, since it only has one step, we only infer in one step. 
For \textit{arithmetic in $F_7$}, we stop the CoT output generation in a step if (i) the output of the step has only one token/element, or (ii) the output of the step is identical to the input of the step. 
For \textit{parity-[2-line]}, we stop the CoT output generation if (i) `?' isn't shown in the output, or (ii) the number of CoT steps is greater than the number of digits. 
For \textit{addition-[1-line]}, \textit{addition-[2-line]} and \textit{addition-[3-line]}, we stop the CoT output generation if (i) `?' and `c' aren't shown in the output, or (ii) the number of CoT steps is greater than the number of digits.
For \textit{multiplication-[1-line]}, we stop the CoT output generation if (i) `+' is not shown in the output, or (ii) the number of CoT steps is greater than the multipliers. 
For \textit{multiplication-[8-line]}, we stop the CoT output generation if (i) `I' is on the left of the first multiplier (e.g. `I' is on the left of $234$ in the example in Table \ref{tab:examples}), or (ii) the number of CoT steps is greater than the multiplication of the number of digits of multipliers. 

For the \textit{arctan} problem, the predicted value is considered correct if the absolute error is smaller than $0.01$. For the other problems, the final output is considered correct only when it is identical to the ground truth. 
For the dataset of each problem, the accuracy is the number of correctly answered instances divided by the total number of instances. 

\textbf{Implementation details:} The \textit{arctan} problem has only $2$ scalars as the input, and the model only has $3$ fully connected layers.  
The model for the \textit{parity} problem, the \textit{arithmetic} problem, the \textit{addition-[1-line]} problem, the \textit{addition-[3-line]} problem, and the \textit{multiplication-[1-line]} problem has 3 Transformer encoders with relative position encoding. 
The \textit{addition-[2-line]} problem and the \textit{multiplication-[8-line]} problem has 6 Transformer encoders, because they have $R = \infty$ but they are $(n, r)$-consistent. The 1st, 3rd, and 5th encoders use relative position encoding, which is designed to acquire information in each $r$-length interval. The 2nd, 4th, and 6th encoders don't use position encoding, which is designed to exchange information of $n$ intervals. 
The optimizer is Adam and the learning rate is $0.0001$. The training data for each task contains $12.8$M CoT steps and was trained for $1$ epoch.

When a CoT formulation has multiple lines, e.g. \textit{multiplication-[8-line]}, each position has multiple tokens. The model first maps the tokens into vectors and concatenates the vectors. Then a fully connected layer maps the concatenated vector into a vector. The remaining model is the same as single-line problems. 

We pad additional empty token ` ' at the beginning and at the end to guarantee that each position can be the central element of a sequence or interval of length $r$, if the problem is $(n, r)$-consistent. Specifically, we pad $[\frac{r - 1}{2}]$ empty tokens at the beginning and $[\frac{r}{2}]$ tokens at the end.



\vspace{2mm}
\subsection{Related Empirical Work}
\label{appendix.related}

Here we review the related empirical work, which includes the evaluation of LLMs in reasoning, chain-of-thoughts for reasoning, dealing with length generalization (LG) in learning to reason, and dealing with LG in text generation. 

\textbf{Evaluations and limitations of LLMs in reasoning}. Continuing with the discussion about evaluations of the reasoning capabilities of LLMs in Sec.~\ref{sec-intro}, we present a more extensive literature survey here. 
In general, evaluations conducted on several latest LLMs showed that they struggled with many reasoning tasks \cite{gendron2023large,tang2023large}. 

In Sec.~\ref{sec-intro}, we discussed empirical works about LG~\cite{anil2022exploring,dziri2023faith,zhang2022unveiling}. In these papers, the authors also tried to mitigate the problem through improved training and CoT~\cite{anil2022exploring}, improved prompting and fine-tuning of LLMs~\cite{zhang2022unveiling}, and curriculum learning~\cite{abbe2023generalization}. An evaluation of the deductive reasoning capability of LLMs was also conducted in~\cite{prystawski2023think}, which shows that CoT helps improve the results, but does not achieve perfect accuracy. None of them studied the LG problem theoretically as we do. Below, we focus on surveying other empirical works. Many of them identified limitations of LLMs in solving different reasoning problems, but few have characterized the limitations in a formal manner to facilitate theoretical investigation.

\cite{meadows2023symbolic} created a dataset specifically for mathematical reasoning that can be perturbed. They showed that perturbations of the tasks heavily affect the results, reducing F1 score from 97\% to 17\%, which suggests that inference is likely to be dominated by surface-level patterns unrelated to the deeper understanding of the mathematical operators. However, this evaluation was done using only BERT~\cite{devlin2018bert} based models, but not on more recent LLMs like ChatGPT and GPT4. \cite{wu2023reasoning} used ``counterfactual'' tasks that deviate from the standard reasoning tasks to evaluate LLMs. It was found that the performance degrades substantially compared to the default conditions, which again suggests that while LLMs can perform reasoning to some extent, they often rely on narrow, non-transferable procedures or surface patterns for task-solving. A counterfactual-based evaluation was also done in \citep{li2023counterfactual}, which reached the same conclusion.   

\cite{liu2023evaluating} evaluated ChatGPT and GPT-4 on logical reasoning. The results showed that they do relatively well on well-known public domain datasets, but their performances drop substantially when newly released and out-of-distribution datasets are used. \cite{xu2023large} also evaluated LLMs using logical reasoning  (deductive, inductive, abductive, and mixed-form reasoning) and gave pros and cons of LLMs. 
\cite{she2023scone} created a dataset for reasoning involving negations and evaluated LLMs and showed poor results. \cite{ando2023evaluating} created a dataset, originally designed for psychological experiments to assess human logical abilities in syllogistic reasoning. 
The authors examined three types of biases observed in human syllogistic reasoning: \textit{belief biases}, \textit{conversion errors}, and \textit{atmosphere effects}. The evaluation on LLMs showed that they struggle with problems involving these biases too. \cite{tan2023towards} created a dataset to evaluate LLMs on temporal reasoning and showed some weaknesses of LLMs. They then proposed an approach to improve the results.   

\textbf{Chain of thoughts (CoT) and variants}. Earlier prompting for solving reasoning problems using LLMs only states the question and the answer. They found that these two pieces of information are insufficient for LLMs to learn to perform effective reasoning. Then \textit{chain of thought} (CoT) prompting~\citep{wei2022chain} was proposed to improve the situation. CoT basically contains
the detailed intermediate reasoning steps between the question and the answer for fine-tuning the LLMs, which significantly enhance LLMs' reasoning capabilities~\cite{chung2022scaling,hsieh2023distilling,mukherjee2023orca,fu2023chain}. \cite{saparov2022language} created a synthetic dataset generated based on first-order logic. They then parsed the generated CoT into symbolic proofs for formal analysis. It was shown that LLMs are capable of reasoning. The success of CoT has encouraged researchers to refine the technique and also propose variations of the technique. 

For example,~\cite{chen2023measuring} proposed a metric to measure the effectiveness of CoT and a technique to improve CoT for vision-language models. \cite{wang2023making} studied using multiple reasoning paths and positive and negative answers to improve CoT reasoning. 
\cite{zhang2023cumulative} proposed cumulative reasoning, which employs LLMs in a cumulative and iterative manner to emulate the human thought process. \cite{qi2023art} proposed a divide-and-conquer algorithm that simulates the self-questioning and recursive thinking process of humans to improve CoT.~\cite{wang2023learning} investigated how to incorporate into relatively small LMs the capabilities of multi-step reasoning and CoT.
\cite{wang2022towards} found that even logically invalid CoT also helps to reason. This was confirmed in~\citep{schaeffer2023invalid}. To deal with unsound inferences, \cite{poesia2023certified} introduced a class of external tools for LLMs called guides that use states and incremental constraints to guide the generation in reasoning. A related work on using external tools was done in~\cite{xu2023rewoo}.
\cite{wang2022self} improved CoT using multiple paths and consistency checks. 
\cite{ling2023deductive} studied the verification of CoT. \cite{stolfo2023understanding} identified part of an LLM responsible for reasoning. In a different direction,~\cite{yang2022chain} argued that the prevailing approach to CoT prompt selection through trial and error is unsatisfactory. They then proposed a principled approach for multi-domain LLM CoT prompt selection.

Several researchers also broadened the CoT method and proposed the neural symbolic \textit{code prompting}~\citep{hu2023code},
\textit{program of thoughts}~\citep{chen2022program,cheng2023binding}, \textit{tree-of-thoughts}~\citep{yao2023tree,long2023large}, \textit{tree-of-mixed-thoughts}~\citep{hu2023tree}, \textit{tree of uncertain thoughts}~\citep{mo2023tree}, \textit{hypergraph-of-thoughts}~\citep{yao2023thinking}, \textit{recursion of thoughts}~\citep{lee2023recursion}, \textit{chain of knowledge}~\citep{wang2023boosting}, \textit{chain of simultaneous thoughts}~\citep{shao2022chaining}, \textit{graph-of-thoughts}~\citep{yao2023beyond}, \textit{faithful chain of thoughts}~\citep{lyu2023faithful}, and thought expansion~\cite{kim-etal-2023-athena}. 
Further, \cite{bi2023program} proposed a complexity measure and chose the optimal complexity to improve the \textit{program of thoughts}~\citep{chen2022program}. \cite{wang2023exploring} proposed a method to improve the generation of equations from natural language questions as the intermediate step to answer the original question. \cite{gao2023pal} combined CoT and Program-Aided Language Models (PAL) for improved reasoning. 

\textbf{Empirical work on LG in reasoning.} Many empirical attempts have been made to modify the Transformer and/or learning biases to better solve the LG problem. \cite{duan2023interpolation} and \cite{zhou2023algorithms} proposed some bias calibration methods to enable the model to learn suitable attention biases. However, their methods are still unable to solve addition perfectly or multiplication at all. \cite{jelassi2023length} proposed to add a small number of long sequences in the training to help solve long sequences, but still could not solve the multiplication problem. 
\cite{chi2023transformer} proposed a Transformer variant with weight-sharing, a working memory, etc, to improve LG for regular languages. It can solve some problems but is still unable to deal with multiplication or addition. Different attention and new architectures are also proposed in \cite{nangia2018listops,bowman2015tree,tay2021long,chowdhury2023monotonic,chowdhury2023recursion}. However, they don't use CoT but only the direct input and output in training. Their methods work on various text copying and list operations but don't solve these problems and do not work on more complex large-number addition and multiplication. Theoretically, learning to reason without intermediate steps (or CoT) in training is not learnable~\cite{wies2023sub,feng2023towards}. Our work needs no specialized architectures for different problems, but just a vanilla Transformer with relative position encoding. 

\textbf{LG in text generation.} LG is also studied for text generation using Transformers, which have the problem when
training on short text while evaluating on longer text \cite{sun2022length,press2021train,ruoss2023randomized,han2023lm}. However, this body of work is very different from the LG problem in reasoning. 

\end{document}